\documentclass[twocolumn]{autart}
\usepackage{cite}
\usepackage{hyperref}
\usepackage{amssymb}
\usepackage{amsmath,amsfonts}
\usepackage{savesym}
\savesymbol{AND}
\usepackage{algorithm,algorithmic}
\usepackage{graphicx}
\usepackage{textcomp}
\usepackage{mathtools, bbm}
\usepackage[mathscr]{euscript}
\usepackage{dsfont}
\usepackage{booktabs}
\usepackage{nicefrac}
\usepackage{xcolor}
\usepackage[font=footnotesize]{subcaption}
\def\proofsq{\mbox{\rule[0pt]{1.3ex}{1.3ex}}}
\newenvironment{proof}{\par\noindent{\bf Proof\ }}{\hfill\proofsq}
\newtheorem{definition}{\bf Definition}
\newtheorem{lemma}{\bf Lemma}
\newtheorem{assumption}{}

\newtheorem{theorem}{\bf Theorem}
\newtheorem{corollary}{\bf Corollary}

\newcommand{\R}{\mathbb{R}}
\newcommand{\E}{\mathbb{E}}

\newcommand{\p}{\mathbb{P}}
\newcommand{\1}{\mathds{1}}
\allowdisplaybreaks
\begin{document}
\begin{frontmatter}
\title{Policy Gradient Methods for Distortion Risk Measures}
\author{Nithia Vijayan}\ead{nithiav@cse.iitm.ac.in},
\author {Prashanth L.A.} \ead{prashla@cse.iitm.ac.in}
\address {Department of Computer Science and Engineering, Indian Institute of Technology Madras, India.}%
\begin{abstract}
We propose policy gradient algorithms which learn risk-sensitive policies in a reinforcement learning (RL) framework.
Our proposed algorithms maximize the distortion risk measure (DRM) of the cumulative reward in an episodic Markov decision process in on-policy and off-policy RL settings, respectively. We derive a variant of the policy gradient theorem that caters to the DRM objective, and integrate it with a likelihood ratio-based gradient estimation scheme. We derive non-asymptotic bounds that establish the convergence of our proposed algorithms to an approximate stationary point of the DRM objective.%
\end{abstract}
\end{frontmatter}
\section{Introduction}
\label{sec:intro}
In a classical reinforcement learning (RL) problem, the objective is to learn a policy that maximizes the mean of the cumulative rewards. But, in many practical applications, we may learn unsatisfactory policies if we only consider the mean. Instead of focusing only on the mean, it is important to consider other aspects of a cumulative reward distribution, viz., variance, shape, and tail probabilities. In literature, a statistical measure, called a risk measure is used to quantify these aspects.

While several risk measures are studied in the literature, there is no consensus on an ideal risk measure. Coherent risk measures are a popular class of risk measures that satisfy desirable properties from a risk aversion viewpoint. In particular, a risk measure is said to be coherent if it is translation invariant, sub-additive, positive homogeneous, and monotonic \cite{artzner99}.
Value-at-Risk (VaR) is a popular risk measure that lacks coherence as it is not sub-additive. Conditional Value-at-Risk (CVaR) \cite{rockafellar2000} is a conditional expectation of outcomes not exceeding VaR, and is a coherent risk measure. However, as suggested in \cite{wang02}, CVaR is not preferable since it treats all outcomes below VaR equally, and ignores those beyond VaR.
Cumulative prospect theory (CPT) \cite{tversky1992advances} is a popular risk measure in human-centered decision making problems. However, CPT is a non-coherent risk measure.
Instead of giving equal focus to all the outcomes, or treating only a fraction of the outcomes using a tail-based risk measure such as CVaR,  it is preferable to consider all outcomes with the right emphasis, while retaining coherency. We describe such a risk measure next.

A family of risk measures called distortion risk measures (DRM) \cite{denneberg1990,wang96} is widely used for optimization in finance and insurance. A DRM uses a distortion function to distort the original distribution, and calculate the mean of the rewards with respect to the distorted distribution. A distortion function allows one to vary the emphasis on each possible reward value. The choice of the distortion function governs the risk measure. A DRM with an identity distortion function is simply the mean of the rewards, while a concave distortion function ensures that the DRM is coherent \cite{wirch03}.
As an aside, the spectral risk functions are equivalent to distortion functions \cite{gzyl06}.
The popular risk measures like VaR and CVaR can be expressed as a DRM using appropriate distortion functions. But, the distortion function is discontinuous for VaR, and though continuous, it is not differentiable at every point for CVaR. As shown in \cite{wang02}, smoothness is a desirable property for a distortion function. In this paper, we focus on smooth distortion functions. Examples include the dual-power function, quadratic function, square-root function, exponential function, and logarithmic function (for additional examples, refer to \cite{jones03, wang96}).

Risk-sensitive RL has been studied widely in the literature, with focus on specific risk measures like expected exponential utility \cite{Borkar02QR}, variance related measures \cite{prashla13}, CVaR \cite{prashanth2014cvar,chow2017risk}, and CPT \cite{prashla16}. In this paper, instead of deriving algorithms that cater to specific risk measures, we consider the whole family of DRMs with smooth distortion functions. The risk-neutral RL approach gives equal importance to all the events, and hence an occasional high/low reward event gets equal priority as all other events. But using DRMs, we can give more emphasis to frequent events, while accounting for infrequent high severity events. As there is no universally accepted ideal risk measure, we may choose a risk measure which best fits our particular problem by picking an appropriate distortion function.

In this paper, we consider a risk-sensitive RL problem, in which an optimal policy is learned by maximizing the DRM of cumulative rewards in an episodic Markov decision process (MDP). We consider this problem in on-policy as well as off-policy settings, and employ the policy gradient solution approach.
%Solving a DRM-sensitive MDP is challenging for two reasons. First, DRM is a risk measure that focuses on the entire distribution of the cumulative reward, while the regular value function objective in a risk-neutral RL setting is concerned with only the mean of this distribution. This observation implies a sample average of the total reward across sample episodes would not be sufficient to estimate DRM. Secondly, a policy gradient algorithm requires an estimate of the gradient of the DRM objective, and such gradient information is not directly available in a typical RL setting. For the risk-neutral case, one has the policy gradient theorem, which leads to a straightforward gradient estimate from sample episodes.
The basis for a policy gradient algorithm is the expression for the gradient of the performance objective. In the risk-neutral case, such an expression is derived using the likelihood ratio (LR) method \cite{sutton_book}. We derive a DRM analogue to the policy gradient theorem.
In the case of DRM, policy gradient estimation is challenging since DRM of a given policy cannot be estimated using a sample mean. We formulate an LR-based estimation using the empirical distribution function (EDF) to approximate DRM, leading to a biased estimate of the DRM policy gradient. In contrast, policy gradient estimation is considerably simpler in a risk-neutral setting as the task is to estimate the mean cumulative reward, and using a sample mean leads to an unbiased gradient estimate.

We characterize the mean squared error (MSE) in DRM policy gradient estimates. In particular, we establish that the MSE is of order $O(\nicefrac{1}{m})$, where $m$ is the batch size (or the number of episodes). Using the DRM policy gradient expression, we propose two policy gradient algorithms which cater to on-policy and off-policy RL settings, respectively. To the best of our knowledge, we are first to derive a policy gradient theorem under a DRM objective, and devise/analyze policy gradient algorithms to optimize DRM in a RL context.

%In our second method, we estimate the DRM from sample episodes using the EDF as a proxy for the true distribution. We provide a non-asymptotic bound on the MSE of this estimator, and this may be of independent interest. In particular, we establish that the mean-square error is of order $O(\nicefrac{1}{m})$. With these DRM estimates, we estimate the DRM gradient, using the smoothed functional (SF) method \cite{katkovnik1972,nesterov2017}, and propose two algorithms which cater to on-policy and off-policy RL settings, respectively. We use a variant of SF which use two function measurements corresponding to two perturbed policies. An SF-based estimation scheme may be restrictive for some applications in an on-policy RL setting, since we need separate sets of episodes corresponding to two perturbed policies. But, in an off-policy RL context, we only need a single set of episodes corresponding to a behavior policy, and this makes our off-policy gradient algorithm practically amenable.

%we establish that an order $O(\frac{1}{\epsilon^2})$ number of iterations of our algorithms are enough to find an $\epsilon$-stationary point.
We provide bounds on the bias and variance of the DRM policy gradient estimates. Using these bounds,
 we establish that our algorithms converge to an approximate stationary point of the DRM objective at a rate of $O(\nicefrac{1}{\sqrt{N}})$. Here $N$ denotes the total number of iterations of the DRM policy gradient algorithm. Our algorithms require $O(\sqrt{N})$ episodes per iteration for both on-policy and off-policy RL settings.

\textit{Related work.}
In \cite{nv23}, the authors develop a general framework for optimization of any smooth risk measure which satisfy certain predefined conditions. It employs a zeroth-order optimization technique, specifically the smooth functional (SF) based estimator for estimating the gradient.
In \cite{tamar2015}, the authors consider a policy gradient algorithm for an abstract coherent risk measure. They derive a policy gradient theorem using the dual representation of a coherent risk measure. Next, using the EDF of the cumulative reward distribution, they propose an estimate of the policy gradient, and this estimation scheme requires solving a convex optimization problem. Finally, they establish asymptotic consistency of their proposed gradient estimate. In \cite{prashla16}, the authors consider a CPT-based objective in an RL setting. They employ a simultaneous perturbation stochastic approximation (SPSA) method for policy gradient estimation, and provide asymptotic convergence guarantees for their algorithm.
In \cite{prashla2021}, the authors survey policy gradient algorithms for optimizing different risk measures in a constrained as well as an unconstrained RL setting. In a non-RL context, the authors in \cite{glynn21} study the sensitivity of DRM using an estimator that is based on the generalized likelihood ratio method, and establish a central limit theorem for their gradient estimator.

In comparison to the aforementioned works, we would like to note the following aspects:\\
(i) For a smooth risk measure, which satisfy certain predefined conditions, \cite{nv23} employs an SF-based gradient estimation scheme. Since our algorithms focus solely on optimizing DRMs, we have derived an equivalent of the policy gradient theorem specifically tailored for DRMs. Furthermore, we have devised a LR-based approach for estimating DRM gradients.
While \cite{nv23} illustrates DRM optimization as an instance within a broader framework, the required number of episodes for achieving $O(\nicefrac{1}{\sqrt{N}})$ convergence varies. The SF-based gradient estimation method employs two trajectories, which might not be practical for all applications, especially those in on-policy RL settings. Conversely, the policy gradient theorem we introduce enables a single trajectory algorithm. Furthermore, our approach illustrates better sample complexity, implying that LR-based methods require fewer episodes to achieve $O(\nicefrac{1}{\sqrt{N}})$ convergence compared to SF-based methods. Our algorithms require $O(\sqrt{N})$ episodes per iteration for both on-policy and off-policy RL scenarios. In contrast, the algorithms proposed in \cite{nv23} necessitate $O(N)$ episodes per iteration for on-policy RL and $O(1)$ episodes per iteration for off-policy RL.\\
(ii) For an abstract coherent risk measure, \cite{tamar2015} uses gradient estimation scheme which requires solving a convex optimization sub-problem, whereas our algorithms can directly estimate the gradient from the samples without solving any optimization sub-problem. Thus our gradient estimation schemes are computationally inexpensive compared to the one in \cite{tamar2015}.\\
(iii) Using the DRM gradient estimate, we analyze policy gradient algorithms, and provide a convergence rate result of order $O(\nicefrac{1}{\sqrt{N}})$. But, the convergence guarantees in \cite{tamar2015} are asymptotic in nature.\\
(iv) In \cite{prashla16}, the guarantees for a policy gradient algorithm based on SPSA are asymptotic in nature, and is for CPT in an on-policy RL setting. CPT is also based on a distortion function, but the distortion function underlying CPT is neither concave nor convex, and hence, it is non-coherent.\\
(v) In \cite{prashla2021}, the authors derive a non-asymptotic bound of $O(\nicefrac{1}{N^{\nicefrac{1}{3}}})$ for an abstract smooth risk measure. They use abstract gradient oracles which satisfies certain bias-variance conditions. In contrast, we provide concrete gradient estimation schemes in RL settings, and our bounds feature an improved rate of $O(\nicefrac{1}{\sqrt{N}})$.

The rest of the paper is organized as follows: Section \ref{sec:pblm} describes the DRM-sensitive episodic MDP. Section \ref{sec:drm} introduces our proposed gradient estimation methods and corresponding algorithms. Section \ref{sec:main} introduces the non-asymptotic bounds for our algorithms, while Section \ref{sec:dpg} offers detailed proofs of convergence. Section \ref{sec:sim} presents empirical results from our proposed algorithms. Finally, Section \ref{sec:conclusions} provides the concluding remarks.
%
%Owing to space limitations, we provide results from an empirical investigation of our proposed algorithms in \cite{arxiv}.
%Section \ref{sec:conv_sf} presents the estimation of the DRM and its gradient using order statistics.
%
\section{Problem formulation}
\label{sec:pblm}
\subsection{Distortion risk measure (DRM)}
\label{subsec:drm}
Let $F_X$ denote the cumulative distribution function (CDF) of a r.v. $X$. The DRM of $X$ is defined using the Choquet integral of $X$ w.r.t. a distortion function $g$ as follows:
\begin{align*}
    %\label{eq:drm}
    \rho_g(X) =\!\!\int_{-\infty}^{0}\!\!\!\!(g(1-F_X(x))-1) dx + \!\int_{0}^{\infty}\!\!\!\!g(1-F_X(x))dx.
\end{align*}
The distortion function $g\!:\![0,1]\!\to\![0,1]$ is non-decreasing, with $g(0)\!=\!0$ and $g(1)\!=\!1$. Some examples of the distortion functions are given in Table \ref{tb:g}, and their plots in Figure \ref{fg:g}.
A distortion function applies varying importance to different segments of the distribution. If we use the identity function as a distortion function, the original distribution remains unaltered. Consequently, the DRM becomes equivalent to the expected value, as demonstrated below.
 \begin{align*}
     %\label{eq:drm}
     \rho_g(X) =\int_{-\infty}^{0} -F_X(x) dx + \int_{0}^{\infty} 1-F_X(x)dx=\E[X].
 \end{align*}
\begin{table}
    \caption{Examples of distortion functions}
    \label{tb:g}
    \begin{center}
        \begin{small}
                \begin{tabular}{ll}
                    \toprule
                    Dual-power function&$g(s)\!=\!1-(1-s)^{\lambda}$, ${\lambda}\geq 2$\\
                    Quadratic function&$g(s)\!=\!(1+{\lambda}) s- {\lambda} s^2$, $0\leq {\lambda}\leq 1$\\
                    Exponential function&$g(s)\!=\!\nicefrac{1-\exp(-{\lambda} s)}{1-\exp(-{\lambda} )}$, ${\lambda}\!>\!0$\\
                    Square-root function&$g(s)\!=\!\nicefrac{\sqrt{1+{\lambda} s}-1}{\sqrt{1+{\lambda}}-1}$, ${\lambda}>0$\\
                    Logarithmic function &$g(s)\!=\!\nicefrac{\log(1+{\lambda} s)}{\log(1+{\lambda} )}$, ${\lambda}>0$\\
                   % {\color{blue}Identity function }&{\color{blue}$g(s)=s$}\\
                    \bottomrule
                \end{tabular}
        \end{small}
    \end{center}
\end{table}%
\begin{figure}
            \caption{Examples of distortion functions}
    \label{fg:g}
    \begin{center}
        \centerline{\includegraphics[width=0.8\columnwidth]{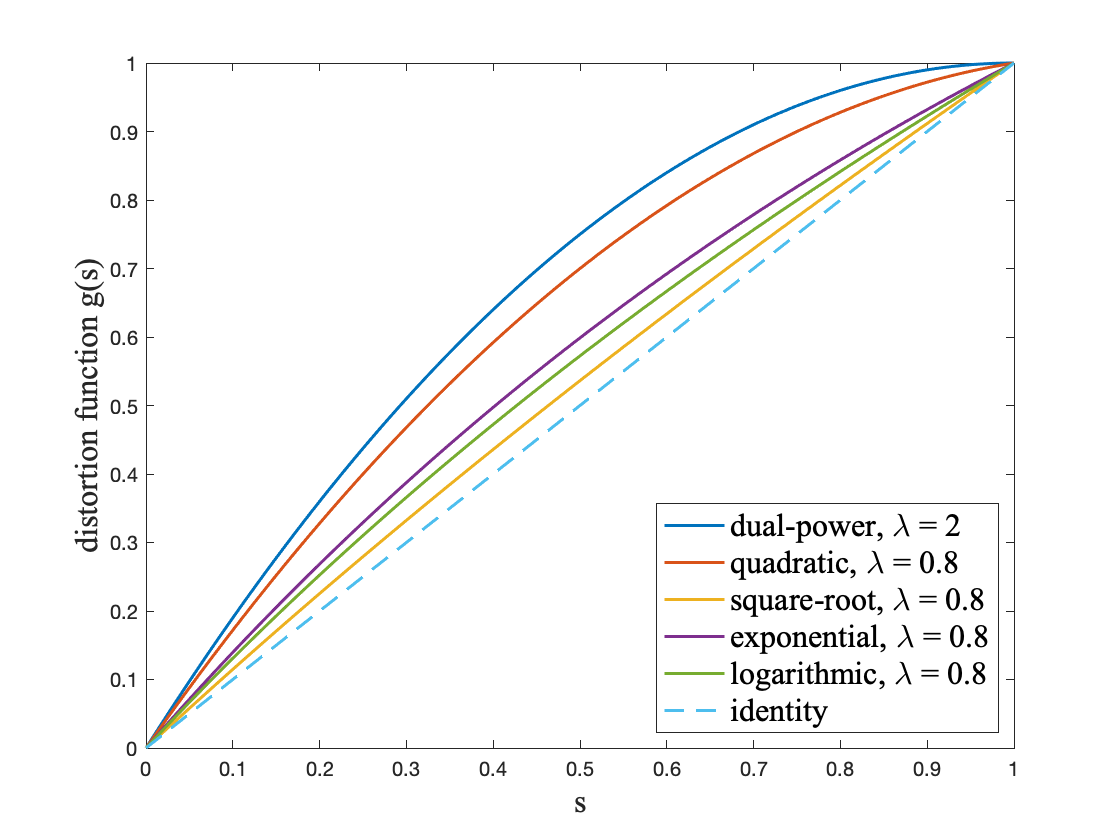}}
    \end{center}
    \vskip -0.2in
\end{figure}

The DRMs are well studied from an `attitude towards risk' perspective, and we refer the reader to \cite{dowd2006,balbas2009properties} for details. In this paper, we focus on `risk-sensitive decision making under uncertainty', with DRM as the chosen risk measure. We incorporate DRMs into a risk-sensitive RL framework, and the following section describes our problem formulation.
%The properties and advantages of the DRM are well studied in actuarial literature, and we refer the reader to \cite{dowd2006,balbas2009properties}. In this paper, we focus on `risk-sensitive decision making under uncertainty', with DRM as the chosen risk measure. The challenges here lie in estimating DRM from samples, and optimizing DRM within a parameterized family using gradient ascent type algorithm.
\subsection{DRM-sensitive MDP}
\label{subsec:drm_mdp}
We consider an MDP with a state space $\mathscr{S}$ and an action space $\mathscr{A}$. We assume that $\mathscr{S}$ and $\mathscr{A}$ are finite spaces. Let $r:\mathscr{S}\times\mathscr{A}\times\mathscr{S}\to [-r_{\textrm{max}},r_{\textrm{max}}], r_{\textrm{max}}\in\mathbb{R}^{+}$ be the scalar reward function, and $p:\mathscr{S}\times\mathscr{S}\times\mathscr{A} \to [0,1]$ be the transition probability function. The actions are selected using parameterized stochastic policies $\{\pi_\theta:\mathscr{S}\times\mathscr{A}\times\mathbb{R}^d\to[0,1],\theta\in\R^d\}$. We consider episodic problems, where each episode starts at a fixed state $S_0$, and terminates at a special absorbing state $0$. We denote by $S_t$ and $A_t$, the state and action at time $t\in\{0,1,\cdots\}$ respectively.
We assume that the policies $\{\pi_\theta,\theta\in\mathbb{R}^d\}$ are proper, i.e., they satisfy the following assumption:
\begin{assumption}
\label{as:proper}
$\exists M\!>\!0:\max\limits_{s\in\mathscr{S}} \mathbb{P}\left(S_M \!\neq\! 0 | S_0\!=\!s,\pi_\theta \right)\!<\!1,\forall \theta \in \R^d$.
\end{assumption}
The assumption \ref{as:proper} is frequently made in the analysis of episodic MDPs (cf. \cite{ndp_book}, Chapter 2).

The cumulative discounted reward $R^\theta$ is defined by
\begin{align}
R^\theta=\sum\nolimits_{t=0}^{T-1}\gamma^t r(S_t,A_t,S_{t+1}),\; \forall \theta \in \mathbb{R}^d,
\end{align}
where $A_t \sim \pi_\theta(\cdot, S_t)$, $S_{t+1}\sim p(\cdot,S_t,A_t)$, $\gamma \in (0,1)$, and $T$ is the random length of an episode.
Notice that $\forall \theta \in \mathbb{R}^d$, $\lvert R^\theta \rvert < \frac{r_{\textrm{max}}}{1-\gamma}$ a.s.
From \ref{as:proper}, we infer that $\E[T]<\infty$. This fact in conjunction with $T\geq0$ implies the following bound:
\begin{align}
    \label{eq:M_pi}
    \exists M_e >0 : T \leq M_e \textrm{ a.s}.
\end{align}
The DRM $\rho_g(\theta)$  is defined by
\begin{align}
\label{eq:rho_g_1}
\rho_g(\theta) = \!\!\int_{-M_r}^{0}\!\!\!\!(g(1\!-\!F_{R^{\theta}}(x))\!-\!1) dx + \!\!\int_{0}^{M_r}\!\!\!\!g(1\!-\!F_{R^{\theta}}(x))dx,
\end{align}
where $F_{R^{\theta}}$ is the CDF of $R^\theta$, and $M_r = \frac{r_{\textrm{max}}}{1-\gamma}$, or any problem specific tight upper bound for $\lvert R^\theta \rvert$.

Our goal is to find a $\theta^*$ which maximizes $\rho_g(\theta)$, i.e.,
\begin{align}
\label{eq:max_theta}
\theta^*\in\textrm{arg}\!\max_{\theta \in \mathbb{R}^d}  \rho_g(\theta).
\end{align}
\section{DRM policy gradient algorithms}
\label{sec:drm}
An iterative gradient-based algorithm can solve \eqref{eq:max_theta} using the following update iteration:
\begin{align}
    \label{eq:theta_update_1}
    \theta_{k+1} = \theta_k + \alpha \nabla \rho_g(\theta_k),
\end{align}
where $\theta_0$ is set arbitrarily, and  $\alpha$ is the step-size. Here $\nabla$ denotes the gradient w.r.t. $\theta$.

But, in a typical RL setting, we do not have direct measurements of the gradient $\nabla \rho_g(\cdot)$.
To overcome this difficulty, we derive a variant of the policy gradient theorem that caters to the DRM objective, and integrate it with an LR-based gradient estimation scheme. In the following sections, we describe our policy gradient algorithms in on-policy and off-policy RL settings, respectively.
%%%%%%%%%%%%%%%%%%%%%%%%%%%%%%%%%%%%%%%%%%%%%%%%%%%%%%%%%%%%%%%
\subsection{DRM policy gradient}
\label{subsec:drm_pg}
In this section, we present a DRM analogue to the policy gradient theorem under the following assumptions:
\begin{assumption}
    \label{as:nabla_logpi}
    $\exists M_d >0: \;\forall \theta\in \R^d, \left\lVert\nabla\log \pi_{\theta}(a| s)\right\lVert \leq M_d, \forall a \in \mathscr{A}, s \in \mathscr{S}$, where $\lVert \cdot \rVert$ is the $d$-dimensional Euclidean norm.
\end{assumption}
\begin{assumption}
    \label{as:g'_bound}
    $\exists M_{g'} >0: \;\forall t\in(0,1), \left\lvert g'(t) \right \rvert \leq M_{g'}$.
\end{assumption}
 The assumptions \ref{as:nabla_logpi}-\ref{as:g'_bound} ensure the boundedness of the DRM policy gradient. An assumption like \ref{as:nabla_logpi} is common to the analysis of policy gradient algorithms (cf. \cite{zhangK2020,papini2018}). A few examples of distortion functions, which satisfy \ref{as:g'_bound} are given in Table \ref{tb:g}.

For deriving a policy gradient theorem variant with DRM as the objective, we first express the CDF $F_{R^{\theta}}(\cdot)$, and its gradient $\nabla F_{R^{\theta}}(\cdot)$  as expectations w.r.t the episodes from the policy $\pi_\theta$. Starting with
\begin{align}
	\label{eq:F_R_pi}
	F_{R^{\theta}}(x) = \E\left[\1\{R^{\theta}\leq x\}\right],
\end{align}
we obtain an expression for $\nabla F_{R^{\theta}}(x)$ in the lemma below, and the proof is available in Section \ref{sec:dpg}.
\begin{lemma}
	\label{lm:nablaFG}
	$\forall x \in(-M_r,M_r)$,
	\begin{align}
		\label{eq:nabla_F_R}
		\nabla F_{R^{\theta}}(x) \!=\! \E\left[\1\{R^{\theta}\!\leq\! x\}\!\sum\nolimits_{t=0}^{T-1}\nabla\log \pi_{\theta}(A_t\mid S_t)\right]
	\end{align}
\end{lemma}
%\begin{proof}
%See Section \ref{subsec:cdf}.
%\end{proof}
We now state the DRM policy gradient theorem below. The reader is referred to Section \ref{sec:dpg} for a proof.
\begin{theorem}(DRM policy gradient)
    \label{thm:nabla_rho_g}
    Assume \ref{as:proper}-\ref{as:g'_bound}. Then the gradient of the DRM in \eqref{eq:rho_g_1} is given by
    \begin{align}
    \label{eq:nabla_rho_g_1}
    \nabla\rho_g(\theta)=-\int_{-M_r}^{M_r} g'(1-F_{R^{\theta}}(x)) \nabla F_{R^{\theta}}(x)dx.
    \end{align}
\end{theorem}
%\begin{proof}
%    See Section \ref{subsec:dpg}.
%\end{proof}
We make the following additional assumptions to ensure the smoothness of the DRM $\rho_g(\theta)$.
\begin{assumption}
    \label{as:nabla_logpi_2}
    $\exists M_{h}>0: \forall \theta\in \R^d, \left\lVert\nabla^2\log \pi_{\theta}(a| s)\right\lVert\leq M_h, \forall a\in \mathscr{A}, s \in \mathscr{S}$, where $\lVert \cdot \rVert$ is the operator norm.
\end{assumption}
\begin{assumption}
    \label{as:g'_bound_2}
    $\exists M_{g''}>0: \;\forall t\in(0,1)$, $ \left\lvert g''(t) \right\rvert \leq M_{g''}$.
\end{assumption}
An assumption like \ref{as:nabla_logpi_2} is common in literature for the non-asymptotic analysis of the policy gradient algorithms (cf. \cite{zhangK2020,shen2019hessian}). A few examples of distortion functions, which satisfy \ref{as:g'_bound_2} are given in Table \ref{tb:g}. Since $g(\cdot)$ is bounded by definition, we can see that any $g(\cdot)$ that satisfies \ref{as:g'_bound_2} will satisfy  \ref{as:g'_bound}. Smoothness assumptions align with (A3) and (A5) can be observed in the literature when optimizing risk measures. For instance, in \cite{tamar2015}, the assumption regarding the risk envelope ensures the smoothness of the constraints. Additionally, in \cite{prashla16}, the authors assume that the CPT weight functions exhibit $H\ddot{o}lder$ continuity, Lipschitz continuity, or local Lipschitz continuity, and that the utility functions are continuous and strictly increasing, or that the first moments are bounded.

%The following results establish that the DRM $\rho_g(\theta)$ is Lipschitz as well as smooth in the parameter $\theta$. The reader is referred to Section \ref{subsec:smooth} for a proof.
The following result establishes that the DRM $\rho_g(\theta)$ is smooth in the parameter $\theta$. The reader is referred to Section \ref{sec:dpg} for a proof.
%\begin{lemma}
%    \label{lm:rho_lip}
%    Assume \ref{as:proper}-\ref{as:g'_bound}. Then, $\forall \theta_1,\theta_2 \in \mathbb{R}^d$,
%    \begin{align*}
%        \left\lvert \rho_g(\theta_1)\!-\! \rho_g(\theta_2)\right\rvert
%        \leq L_\rho \left\lVert \theta_1 \!-\! \theta_2 \right\rVert,
%        L_\rho=2M_rM_{g'}M_eM_d.
%    \end{align*}
%\end{lemma}
%\begin{proof}
%    See Section \ref{subsec:smooth}.
%\end{proof}
\begin{lemma}
    \label{lm:nabla_rho_lip}
    Assume \ref{as:proper}-\ref{as:g'_bound_2}. Then $\forall \theta_1,\theta_2 \in \mathbb{R}^d$,
    \begin{align*}
      &\left\lVert \nabla\rho_g(\theta_1) \!-\! \nabla \rho_g(\theta_2) \right\rVert \leq
        L_{\rho'} \left\lVert  \theta_1  \!-\! \theta_2\right\rVert,\\
    &\textrm{where }
    L_{\rho'}=2M_r M_e \left( M_h M_{g'}+M_eM_d^2 (M_{g'}+ M_{g''})\right).
        \end{align*}
\end{lemma}
%\begin{proof}
%    See Section \ref{subsec:smooth}.
% \end{proof}
%%%%%%%%%%%%%%%%%%%%%%%%%%%%%%%%%%%%%%%%%%%%%%%%%%%%%%%%%%%%%%%
\subsection{DRM optimization}%: LR-based gradient estimation}
\label{subsec:lr}
%We use DRM policy gradient theorem to form an estimate of the DRM policy gradient.
In the following sections, we describe gradient algorithms that use \eqref{eq:nabla_rho_g_1} to derive DRM gradient estimates.
\subsubsection{On-policy DRM optimization}
\label{subsubsec:lr_onpolicy}
% \begin{proof}
%     See Section \ref{subsec:cdf}.
% \end{proof}
 %We replace the expectations in \eqref{eq:F_R_pi} and \eqref{eq:nabla_F_R} by sample averages, and forms an estimator for the CDF and its gradient.
%\begin{align}
%    \label{eq:nabla_F_R}
%    \nabla F_{R^{\theta}}(x) \!=\! \E\left[\1\{R^{\theta}\leq x\}\sum\nolimits_{t=0}^{T-1}\nabla\log \pi_{\theta}(A_t\mid S_t)\right].
%\end{align}
%The reader is referred to Lemma \ref{lm:nablaFG} in Section \ref{sec:conv_lr} for a proof.
We generate $m$ episodes using the policy $\pi_\theta$, and estimate $F_{R^{\theta}}(\cdot)$ and $\nabla F_{R^{\theta}}(\cdot)$ using sample averages.
We denote by $R^{\theta}_i$ the cumulative reward, and $T^i$ the length of the episode $i$. Also, we denote by  $A_t^i$ and $S_t^i$ the action and state at time $t$ in episode $i$, respectively.
Let $G^m_{R^{\theta}}(\cdot)$ denote the EDF of $F_{R^{\theta}}(\cdot)$, and is defined by
\begin{align}
\label{eq:G}
G^m_{R^{\theta}}(x) = \frac{1}{m}\sum\limits_{i=1}^m \1\{R^{\theta}_i\leq x\}.
\end{align}
We form the estimate $\widehat{\nabla}G^m_{R^{\theta}}(\cdot)$ of $\nabla F_{R^{\theta}}(\cdot)$ as follows:
\begin{align}
 \label{eq:nabla_G}
 \widehat{\nabla}G^m_{R^{\theta}}(x) \!=\! \frac{1}{m}\!\sum\limits_{i=1}^m\1\{R^{\theta}_i\!\leq\! x\}\!\!\sum\limits_{t=0}^{T^i-1}\nabla\log \pi_{\theta}(A_t^i | S_t^i).
\end{align}
Using the estimates from \eqref{eq:G} and \eqref{eq:nabla_G}, we estimate the gradient $\nabla\rho_g(\theta)$ in \eqref{eq:nabla_rho_g_1} as follows:
\begin{align}
    \label{eq:hat_nabla_rho_G}
    \widehat{\nabla}^G\rho_g(\theta)=-\int_{-M_r}^{M_r} g'(1-G^m_{R^{\theta}}(x)) \widehat{\nabla}G^m_{R^{\theta}}(x)dx.
\end{align}
Using order statistics of $m$ samples $\{R^\theta_i\}_{i=1}^m$, we can compute the integral in \eqref{eq:hat_nabla_rho_G} as given in the lemma below.
%The reader is referred to \cite{arxiv} for a proof.
The reader is referred to Appendix \ref{sec:conv_aux} for a proof.
\begin{lemma}
    \label{lm:hat_nabla_rho_G}
    \begin{align}
        \label{eq:hat_nabla_rho_G1}
        &\widehat{\nabla}^G\rho_g(\theta)\!=\!\frac{1}{m}\sum_{i=1}^{m-1} \left( R^\theta_{(i)}\!-\!R^\theta_{(i+1)}\right) g'\left(1-\frac{i}{m}\right) \sum_{j=1}^i \nabla l^{\theta}_{(j)}\nonumber\\
        &\quad+ \frac{1}{m}\left(R^\theta_{(m)} -M_r\right) g'_{+}(0)\sum\nolimits_{j=1}^m\nabla l^{\theta}_{(j)}.
    \end{align}
\end{lemma}
%\begin{proof}
%    See Appendix \ref{subsec:pf_hat_nabla_rho_G}.
%\end{proof}
%\begin{align}
%    \label{eq:hat_nabla_rho_G1}
%        \widehat{\nabla}^G\!\rho_g(\theta)&\!=\!\frac{1}{m}\sum_{i=1}^{m-1} \left( R^\theta_{(i)}\!-\!R^\theta_{(i+1)}\right) g'\left(1\!-\!\frac{i}{m}\right) \sum_{j=1}^i \nabla l^{\theta}_{(j)}\nonumber\\
%    &\quad+ \frac{1}{m}\left(R^\theta_{(m)} -M_r\right) g'_{+}(0)\sum\nolimits_{j=1}^m\nabla l^{\theta}_{(j)}.
%%    \widehat{\nabla}^G\!\rho_g(\theta)%\nonumber\\
%%    &= \frac{1}{m}\sum_{i=1}^{m-2} \left( R^\theta_{(i)}\!-\!R^\theta_{(i+1)}\right) g'\!\left(1\!-\!\frac{i}{m}\right) \!\sum_{j=1}^i \nabla l^{\theta}_{(j)}\nonumber \\
%%    &\quad-  \frac{1}{m} M_r g'(0)\sum_{j=1}^m \nabla l^{\theta}_{(j)}.
%\end{align}
In the above, $R^\theta_{(i)}$ is the $i^{th}$ smallest order statistic from the samples $\{R^\theta_i\}_{i=1}^m$, and\\ $\nabla l^{\theta}_{(i)}= \sum_{t=0}^{T^{(i)}-1}\nabla\!\log \pi_{\theta}(A_t^{(i)} | S_t^{(i)})$, with $T^{(i)}$ denoting the length, and $S_t^{(i)}$ and $A_t^{(i)}$ the state and action at time $t$ of the episode corresponding to $R^\theta_{(i)}$. Here $g'_{+}(0)$ is the right derivative of the distortion function $g$ at $0$. %See Lemma \ref{lm:hat_nabla_rho_G} in Section \ref{sec:conv_aux} for a proof.

The gradient estimator in \eqref{eq:hat_nabla_rho_G} is biased since $\E[g'(1-G^m_{R^{\theta}}(\cdot))]\neq g'(1-F_{R^{\theta}}(\cdot))$. However, the bias can be controlled by increasing the number of episodes $m$. A bound for the MSE of this estimator is given below. The reader is referred to Section \ref{sec:dpg} for a proof.
\begin{lemma}
    \label{lm:biasG}
    \begin{align*}
        \E\!\left\lVert\widehat{\nabla}^G \rho_g(\theta) \!-\! \nabla\rho_g(\theta)\right\rVert^2 \leq \frac{32 M_r^2M_e^2M_d^2(e^2M_{g'}^2 + M_{g''}^2)}{m}.
    \end{align*}
\end{lemma}
%\begin{proof}
%    See Section \ref{subsec:onp_lr}.
%\end{proof}
%\begin{lemma}
%    \label{lm:varG}
%    $\E\left[\left\lVert \widehat{\nabla}^G \rho_g(\theta) \right\rVert^2\right] \leq 4 M_r^2 M_{g'}^2 M_e^2M_d^2.$
%\end{lemma}
%\begin{proof}
%    See Section \ref{subsec:onp_lr}.
% \end{proof}
We solve \eqref{eq:max_theta} using the following update iteration:
\begin{align}
\label{eq:approx_theta_update_G_lr}
\theta_{k+1} = \theta_k + \alpha  \widehat{\nabla}^G\!\rho_g(\theta_k).
\end{align}
Algorithm \ref{alg:drm_onP_lr} presents the pseudocode of DRM-OnP-LR.
\begin{algorithm}[t]
    \caption{DRM-OnP-LR}
    \label{alg:drm_onP_lr}
    \begin{algorithmic}[1]
        \STATE \textbf{Input}: Parameterized form of the policy $\pi_\theta$, iteration limit $N$, step-size $\alpha$, and batch size $m$;
        \STATE \textbf{Initialize}: Policy $\theta_{0} \!\in\! \R^d$, discount factor $\gamma \!\in\! (0,1)$;
        \FOR {$k=0,\hdots, N-1$ }
        \STATE Generate $m$ episodes each using $\pi_{\theta_k}$;
        \STATE Use \eqref{eq:hat_nabla_rho_G1} to estimate $\widehat{\nabla}^{G}\!\rho_g(\theta_k)$;
        \STATE Use \eqref{eq:approx_theta_update_G_lr} to calculate $\theta_{k+1}$;
        \ENDFOR
        \STATE \textbf{Output}: Policy $\theta_R$, where $R\sim\mathcal{U}\{0, N-1\}$.%where $R$ is chosen uniformly at random from $\{1,\cdots,N\}$.
    \end{algorithmic}
\end{algorithm}
%%%%%%%%%%%%%%%%%%%%%%%%%%%%%%%%%%%%%%%%%%%%%%%%%%%%%%%%%%%%%%%%%%%%%%%%%%%%%%%%%%%%%%%%%%%%%%%%%%%%%
\subsubsection{Off-policy DRM optimization}
\label{subsubsec:lr_offpolicy}
In an off-policy RL setting, we optimize the DRM of $R^\theta$ from the episodes generated by a behavior policy $b$, using the importance sampling (IS) ratio. We require the behavior policy $b$ to be proper, i.e.,
\begin{assumption}
    \label{as:b_proper}
$\exists M > 0:\; \max_{s\in\mathscr{S}}\mathbb{P}\!\left(S_M \neq 0 \mid S_0=s, b \right)<1$.
\end{assumption}
We also assume that the target policy $\pi_\theta$ is absolutely continuous w.r.t. the behavior policy $b$, i.e.,
\begin{assumption}
 \label{as:b_pol}
$\forall \theta \! \in \R^d, b(a | s)\!=\!0 \Rightarrow \pi_\theta(a | s)\!=\!0,\forall a \in \mathscr{A}, \forall s \!\in \mathscr{S}$.
\end{assumption}
Assumption \ref{as:b_pol} is standard in an off-policy RL setting (cf. \cite{sutton08}).

The cumulative discounted reward $R^b$ is defined by
\begin{align}
\label{eq:Rb}
R^b=\sum\nolimits_{t=0}^{T-1}\gamma^t r(S_t,A_t,S_{t+1}),
\end{align}
where $A_t \sim b(\cdot, S_t)$, $S_{t+1}\sim p(\cdot,S_t,A_t)$, $\gamma \in (0,1)$, and $T$ is the random length of an episode. As before, \ref{as:b_proper} implies $\E[T]<\infty$ and the following bound:
    \begin{align}
            \label{eq:M_b}
            \exists M_e >0 : T \leq M_e, \textrm{ a.s}.
        \end{align}
The importance sampling ratio $\psi^\theta$ is defined by
\begin{align}
\label{eq:psi}
\psi^\theta = \prod\nolimits_{t=0}^{T-1}\frac{\pi_{\theta}(A_t\mid S_t)}{b(A_t\mid S_t)}.
\end{align}
    From \ref{as:nabla_logpi} and \ref{as:b_pol}, we obtain $\forall \theta \in \R^d,\pi_{\theta}(a|s)>0$ and $b(a|s) >0$, $\forall a \in \mathscr{A}, \textrm{ and } \forall s \in \mathscr{S}$. This fact in conjunction with \eqref{eq:M_b} implies the following bound for $\psi^\theta$:
    \begin{align}
     \label{eq:is_ratio}
    \exists M_s>0 : \forall \theta\in\R^d, \psi^\theta \leq M_s, \textrm{ a.s}.
    \end{align}
We express the CDF $F_{R^{\theta}}(\cdot)$, and its gradient $\nabla F_{R^{\theta}}(\cdot)$  as expectations w.r.t. the episodes from the policy $b$. Starting with
\begin{align}
    \label{eq:F_R_b}
   F_{R^{\theta}}(x) = \E\left[\1\{R^b\leq x\}\psi^\theta  \right],
   \end{align}
we obtain the following analogue of Lemma \ref{lm:nablaFG} for the off-policy case. %The reader is referred to \cite{arxiv} for a proof.
The reader is referred to Appendix \ref{sec:conv_lr_offp} for a proof.
%Since $F_{R^{\theta}}(x) = \E\left[\1\{R^b\!\leq\! x\}\psi^\theta  \right]$, we obtain the following analogue of \eqref{eq:nabla_F_R} for the off-policy setting%\footnote{See Lemma \ref{lm:nablaFH} in Appendix \ref{sec:cdf} for a proof.}
\begin{lemma}
    \label{lm:nablaFH}
    $\forall x \in(-M_r,M_r)$,
    \begin{align}
        \label{eq:nabla_F_R_b}
        \nabla F_{R^{\theta}}(x) = \E\!\left[\1\{R^b\!\leq\! x\}\psi^\theta  \sum\limits_{t=0}^{T-1}\nabla\log \pi_{\theta}(A_t|S_t)\right]
    \end{align}
\end{lemma}
%\begin{proof}
%See Section \ref{subsec:offp_lr}.
%    \end{proof}
%\begin{align}
%    \label{eq:nabla_F_R_b}
%    \nabla F_{R^{\theta}}(x) \!=\! \E\left[\1\{R^b\!\leq\! x\}\psi^\theta \! \sum\nolimits_{t=0}^{T-1}\!\nabla\log \pi_{\theta}(A_t|S_t)\right].
%\end{align}
%See Lemma \ref{lm:nablaFH} in Section \ref{sec:conv_lr} for a proof.
We generate $m$ episodes using the policy $b$ to estimate $F_{R^{\theta}}(\cdot)$ and $\nabla F_{R^{\theta}}(\cdot)$ using sample averages. We denote by $R^b_i$ the cumulative reward, $\psi^\theta_i$ the IS ratio.%, and $T^i$ the length of the episode $i$. Also, we denote by $A_t^i$ and $S_t^i$ the action and state at time $t$ in episode $i$, respectively.

We form the estimate $H^m_{R^{b}}(\cdot)$ of $F_{R^{\theta}}(\cdot)$ as follows:
\begin{align}
    &H^m_{R^{\theta}}(x) = \min\{\hat{H}^m_{R^{\theta}}(x),1\}, \textrm{ where} \label{eq:H}\\
    &\hat{H}^m_{R^{\theta}}(x) = \frac{1}{m}\sum\nolimits_{i=1}^m\1\{R^{b}_i\leq x\}\psi^\theta_i.\label{eq:hatH}
\end{align}
The importance sampling ratio in \eqref{eq:hatH} can set $\hat{H}^m_{R^{\theta}}(x)$ a value above $1$. Since we are estimating a CDF, we restrict $\hat{H}^m_{R^{\theta}}(x)$ to one in $H^m_{R^{\theta}}(x)$. %From Lemma \ref{lm:biasH} in Section \ref{sec:conv_lr}, we can see that the aforementioned restriction would not affect the bound on the bias of the DRM gradient estimation.

We form the estimate $\widehat{\nabla}H^m_{R^{b}}(\cdot)$ of $\nabla F_{R^{\theta}}(\cdot)$ as follows:
\begin{align}
\label{eq:nabla_H}
\widehat{\nabla}H^m_{R^{\theta}}(x) \!=\! \frac{1}{m}\!\sum\limits_{i=1}^m\!\1\{R^{b}_i \!\leq\! x\} \psi^\theta_i\!\sum\limits_{t=0}^{T^i\!-\!1} \!\!\nabla\log \pi_{\theta}(A_t^i | S_t^i)\!
\end{align}
Using \eqref{eq:H} and \eqref{eq:nabla_H}, we estimate $\nabla\rho_g(\theta)$ by
\begin{align}
    \label{eq:hat_nabla_rho_H}
    \widehat{\nabla}^{H}\!\rho_g(\theta)=-\int_{-M_r}^{M_r} g'(1-H^m_{R^{\theta}}(x)) \widehat{\nabla}H^m_{R^{\theta}}(x)dx.
\end{align}

As in the on-policy case, the integral in \eqref{eq:hat_nabla_rho_H} can be computed using order statistics of the samples $\{R^b_i\}_{i=1}^m$, as given in the lemma below. %The reader is referred to \cite{arxiv} for a proof.
The reader is referred to Appendix \ref{sec:conv_aux} for a proof.
\begin{lemma}
    \label{lm:hat_nabla_rho_H}
    \begin{align}
        \label{eq:hat_nabla_rho_H1}
        &\widehat{\nabla}^H\!\rho_g(\theta)
        =\frac{1}{m}\sum\nolimits_{i=1}^{m-1} \left(\left( R^b_{(i)}-R^b_{(i+1)}\right) \right.\nonumber\\
        &\quad\left.\times g'\left(1\!-\! min\left\{1,\frac{1}{m}\sum\nolimits_{j=1}^{i}\psi^\theta_{(j)}\right\}\right)\! \sum\nolimits_{j=1}^i \!\nabla l^{\theta}_{(j)}\psi^\theta_{(j)}\right)\nonumber\\
        &\quad+  \frac{1}{m}\left(R^b_{(m)} -M_r\right) g'_{+}(0)\sum\nolimits_{j=1}^m\nabla l^{\theta}_{(j)}\psi^\theta_{(j)}.
    \end{align}
\end{lemma}
%\begin{proof}
%    See Appendix \ref{subsec:pf_hat_nabla_rho_H}.
%    \end{proof}
%\begin{align}
%    \label{eq:hat_nabla_rho_H1}
%        &\widehat{\nabla}^H\!\rho_g(\theta)
%        =\frac{1}{m}\sum\nolimits_{i=1}^{m-1} \left(\left( R^b_{(i)}-R^b_{(i+1)}\right) \right.\nonumber\\
%        &\qquad\left.\times g'\left(1\!-\! min\left\{1,\frac{1}{m}\sum\nolimits_{j=1}^{i}\psi^\theta_{(j)}\right\}\right) \sum\nolimits_{j=1}^i \nabla l^{\theta}_{(j)}\psi^\theta_{(j)}\right)\nonumber\\
%        &\quad+  \frac{1}{m}\left(R^b_{(m)} -M_r\right) g'_{+}(0)\sum\nolimits_{j=1}^m\nabla l^{\theta}_{(j)}\psi^\theta_{(j)}.
%%        &\widehat{\nabla}^H\!\rho_g(\theta)\nonumber\\
%%        &= \frac{1}{m}\sum_{i=1}^{m-2} \left( R^b_{(i)}-R^b_{(i+1)}\right) \nonumber\\
%%        &\qquad\qquad g'\left(1- min\left\{1,\frac{1}{m}\sum_{j=1}^{i}\psi^\theta_{(j)}\right\}\right) \sum_{j=1}^i \nabla l^{\theta}_{(j)}\psi^\theta_{(j)} \nonumber\\
%%        &\quad-  \frac{1}{m} M_r g'(0)\sum_{j=1}^m \nabla l^{\theta}_{(j)}\psi^\theta_{(j)}.
%\end{align}
In the above, $R^b_{(i)}$ is the $i^{th}$ smallest order statistic from the samples $\{R^b_i\}_{i=1}^m$, and $\psi^\theta_{(i)}$ is the importance sampling ratio corresponding to $R^b_{(i)}$. Also, $\nabla l^{\theta}_{(i)}= \sum_{t=0}^{T^{(i)}-1}\nabla\!\log \pi_{\theta}(A_t^{(i)} | S_t^{(i)})$, with $T^{(i)}$ denoting the length, and $S_t^{(i)}$ and $A_t^{(i)}$ are the state and action at time $t$ of the episode corresponding to $R^b_{(i)}$.

As in the on-policy case, the estimator in \eqref{eq:hat_nabla_rho_H} is biased, but can be controlled by increasing the number of episodes $m$. A bound on the MSE of our estimator is given below.
%The reader is referred to \cite{arxiv} for a proof.
The reader is referred to Section \ref{subsec:offp_lr} for a proof.
\begin{lemma}
    \label{lm:biasH}
    \begin{align*}
        &\E\left\lVert\widehat{\nabla}^H \rho_g(\theta) -\nabla\rho_g(\theta)\right\rVert^2\\
        &\qquad\leq \frac{32 M_r^2M_s^2M_e^2M_d^2(e^2M_{g'}^2 + M_{g''}^2M_s^2)}{m}.
    \end{align*}
\end{lemma}
%\begin{proof}
%    See Section \ref{subsec:offp_lr}.
%    \end{proof}
%\begin{lemma}
%    \label{lm:varH}
%    $\E\left[\left\lVert \widehat{\nabla}^H\! \rho_g(\theta) \right\rVert^2\right] \leq 4 M_r^2 M_{g'}^2 M_s^2M_e^2M_d^2.$
%\end{lemma}
%\begin{proof}
%     See Section \ref{subsec:offp_lr}.
%\end{proof}
We solve \eqref{eq:max_theta} using the following update iteration:
\begin{align}
    \label{eq:approx_theta_update_H_lr}
    \theta_{k+1} = \theta_k + \alpha  \widehat{\nabla}^{H}\!\rho_g(\theta_k).
\end{align}
The pseudocode of DRM-OffP-LR algorithm is similar to Algorithm \ref{alg:drm_onP_lr}, except that we generate episodes using the policy $b$, and use \eqref{eq:hat_nabla_rho_H1} and  \eqref{eq:approx_theta_update_H_lr} in place of \eqref{eq:hat_nabla_rho_G1} and \eqref{eq:approx_theta_update_G_lr},  respectively.
%Algorithm \ref{alg:drm_offP} presents the pseudocode of DRM-offP algorithm.
%\begin{algorithm}[h]
%    \caption{DRM-offP}
%    \label{alg:drm_offP}
%    \begin{algorithmic}[1]
%        \STATE \textbf{Input}: Parameterized form of the policy $\pi$, the policy b, iteration limit $N$, step-size $\alpha$, and batch size $m$;
%        \STATE \textbf{Initialize}: Target policy parameter $\theta_{0} \in \mathbb{R}^d$, and the discount factor $\gamma \in (0,1)$;
%        \FOR {$k=0,\hdots, N-1$ }
%        \STATE Generate $m$ episodes each using $b$;
%        \STATE Use \eqref{eq:hat_nabla_rho_H} to estimate $\widehat{\nabla}^{H}\!\rho_g(\theta_k)$;
%        \STATE Use \eqref{eq:approx_theta_update_H} to calculate $\theta_{k+1}$;
%        \ENDFOR
%        \STATE \textbf{Output}: Policy $\theta_R$, where $R$ is chosen uniformly at random from $\{1,\cdots,N\}$.
%    \end{algorithmic}
%\end{algorithm}
%%%%%%%%%%%%%%%%%%%%%%%%%%%%%%%%%%%%%%%%%%%%%%%%%%%%%%%%%%%%%%%%%%%%%%%%%%%

%
\section{Main results}
\label{sec:main}
%We make the following assumptions to derive the DRM analogue to the policy gradient theorem:
%\begin{assumption}
%    \label{as:proper}
%    The class of target policies $\{\pi_\theta, \theta \in \Theta'\}$ is proper,
%    i.e.,  there exists a positive constant $M$ s.t.
%    $\forall \theta \in \Theta,\; \max_{s\in\mathscr{S}}\;\mathbb{P}\left(S_M \neq 0 \mid S_0=s,\pi_\theta \right)<1$.
%\end{assumption}
%\begin{assumption}
%    \label{as:nabla_logpi}
%    $\forall \theta\in \R^d, \left\lVert\nabla\log \pi_{\theta}(a\mid s)\right\lVert\leq M_l, \forall a\in \mathscr{A}, s \in \mathscr{S}$.
%\end{assumption}
%\begin{assumption}
%    \label{as:g'_bound}
%    $\exists M_g >0$, $g'_{+}(0) \leq M_g$, where $g'_{+}(0)$ is the right derivative of the distortion function $g$ at $0$.
%\end{assumption}
%The assumption \ref{as:proper} is a common requirement in the analysis of episodic MDPs (cf. \cite{ndp_book}), while the assumptions \ref{as:nabla_logpi}-\ref{as:g'_bound} are required to ensure the boundedness of the gradient estimates.
%
%The DRM analogue to the policy gradient theorem is presented below.
%\begin{theorem}(\textbf{DRM policy gradient})
%    \label{thm:nabla_rho_g}
%    Assume \ref{as:proper}-\ref{as:g'_bound} (see Section \ref{sec:main}). Then the gradient of the DRM in \eqref{eq:rho_g_1} is
%    \begin{align*}
%        \nabla\rho_g(\theta)=-\!\int_{-M_r}^{M_r} g'(1-F_{R^{\theta}}(x)) \nabla F_{R^{\theta}}(x)dx.
%    \end{align*}
%\end{theorem}
%\begin{proof}
%    See Section \ref{sec:conv}.
%\end{proof}
Our non-asymptotic analysis establishes a bound on the number of iterations of our algorithms to find an $\epsilon$-stationary point of the DRM, which is defined below.
\begin{definition}(\textbf{$\epsilon$-stationary point})
\label{def:esolution}
Let $\theta_R$ be the output of an algorithm. Then, $ \theta_R $ is called an $ \epsilon$-stationary point of problem \eqref{eq:max_theta}, if $\,\E\left\Vert \nabla \rho_g \left( \theta_R \right) \right\rVert^2 \leq \epsilon$.
\end{definition}

In an RL setting, the DRM objective need not be convex. Hence, we establish the convergence of  our proposed algorithms to an $\epsilon$-stationary point. Such an approach is common in the risk-neutral setting as well, cf. \cite{papini2018,shen2019hessian}.

We derive convergence rate of our algorithms for a random iterate $\theta_R$, that is chosen uniformly at random from the policy parameters $\{\theta_0,\cdots,\theta_{N-1}\}$. %Such a randomized stochastic gradient algorithm  has been studied earlier in an stochastic optimization setting in \cite{ghadimi2013}.
%We make the following additional assumptions for our analysis:
%\begin{assumption}
%    \label{as:nabla_logpi_2}
%    $\exists M_{h}>0$ such that $\forall \theta\in \R^d, \forall a\in \mathscr{A}, s \in \mathscr{S}$, $\left\lVert\nabla^2\log \pi_{\theta}(a\mid s)\right\lVert\leq M_h$, where $\lVert \cdot \rVert$ is the operator norm.
%\end{assumption}
%\begin{assumption}
%    \label{as:g'_bound_2}
%    $\exists M_{g''}>0$ such that $\forall t\in(0,1)$, $ \left\lvert g''(t) \right\rvert \leq M_{g''}$.
%\end{assumption}
%The assumptions \ref{as:nabla_logpi_2}-\ref{as:g'_bound_2} are needed to ensure the smoothness of the DRM $\rho_g$. An assumption like \ref{as:nabla_logpi_2} is common in literature for the non-asymptotic analysis of the policy gradient algorithms (cf. \cite{zhangK2020,shen2019hessian}). A few examples of distortion functions, which satisfy \ref{as:g'_bound_2} are given in Table \ref{tb:g}. Since $g(\cdot)$ is bounded by definition, we can see that any $g(\cdot)$ that satisfies \ref{as:g'_bound_2} will satisfy  \ref{as:g'_bound} also.
%\subsection{Non-asymptotic bounds (LR-based scheme)}
We provide a convergence rate for the algorithm DRM-OnP-LR and DRM-OffP-LR below. %The proofs are available in Section \ref{sec:dpg} and  \cite{arxiv}, respectively.
The proofs are available in Section \ref{sec:dpg} and  Appendix \ref{sec:conv_lr_offp}, respectively.
\begin{theorem}(DRM-OnP-LR)
    \label{tm:drm_onP_lr}
    Assume \ref{as:proper}-\ref{as:g'_bound_2}.  Let $\{\theta_i\}_{i=0}^{N-1}$ be the policy parameters generated by DRM-OnP-LR using \eqref{eq:approx_theta_update_G_lr}, and let $\theta_R$ be chosen uniformly at random from this set. Then,
    %Set $\alpha=\frac{1}{\sqrt{N}}$, and $m=\sqrt{N}$. Then,
    \begin{align}
       \label{eq:tm_drm_onP_lr}
        &\E\left\lVert \nabla \rho_g(\theta_R)\right\rVert^2
        \leq \frac{2 \left(\rho_g^* - \rho_g(\theta_{0})\right)}{N \alpha} \nonumber\\
        & + 4 M_r^2M_e^2M_d^2 \left(\alpha M_{g'}^2L_{\rho'} + \frac{8(e^2M_{g'}^2 + M_{g''}^2)}{m}\right).
       \end{align}
%    \begin{align}
%        \label{eq:tm_drm_onP_lr}
%        &\E\left[\left\lVert \nabla \rho_g(\theta_R)\right\rVert^2\right]
%        \leq \frac{2 \left(\rho_g^* - \rho_g(\theta_{0})\right)}{\sqrt{N}}\nonumber\\
%        &\quad+ 4 M_r^2M_e^2M_d^2 \left(\frac{M_{g'}^2L_{\rho'}}{\sqrt{N}} + \frac{8(e^2M_{g'}^2 + M_{g''}^2)}{m}\right).
%    \end{align}
%    \begin{align*}
%        &\E\left[\left\lVert\nabla\rho_g(\theta_R)\right\rVert^2\right]\leq \frac{2 \left(\rho_g^* - \rho_g(\theta_{0})\right)}{\sqrt{N}} \\
%        &\qquad+ \frac{4 M_r^2M_e^2M_d^2 \left(M_{g'}^2L_{\rho'} + 8(e^2M_{g'}^2 + M_{g''}^2)\right)}{\sqrt{N}}.
%    \end{align*}
    In the above, $\rho_g^*=\max_{\theta\in\mathbb{R}^d}\rho_g(\theta)$, $M_r = \frac{r_{\textrm{max}}}{1-\gamma}$, and $L_{\rho'}$ is as in Lemma \ref{lm:nabla_rho_lip}.
    The constants $M_d, M_{g'}, M_{g''}$, and $M_h$ are as defined in \ref{as:nabla_logpi}-\ref{as:g'_bound_2}, while $M_e$ is an upper bound on the episode length from \eqref{eq:M_pi}.
\end{theorem}
%\begin{proof}
%    See Section \ref{subsec:onp_lr}.
%    \end{proof}
We specialize the bound in \eqref{eq:tm_drm_onP_lr} to a particular choice of step-size $\alpha$, and batch size $m$ in the corollary below.
\begin{corollary}
    \label{cr:drm_onP_lr}
    Set $\alpha=\frac{1}{\sqrt{N}}$, and $m=\sqrt{N}$. Then, under the conditions of Theorem \ref{tm:drm_onP_lr}, we have
    \begin{align*}
        &\E\left[\left\lVert\nabla\rho_g(\theta_R)\right\rVert^2\right]\leq \frac{2 \left(\rho_g^* - \rho_g(\theta_{0})\right)}{\sqrt{N}} \\
        &\qquad+ \frac{4 M_r^2M_e^2M_d^2 \left(M_{g'}^2L_{\rho'} + 8(e^2M_{g'}^2 + M_{g''}^2)\right)}{\sqrt{N}}.
    \end{align*}
\end{corollary}
\begin{theorem}(DRM-OffP-LR)
\label{tm:drm_offP_lr}
Assume \ref{as:proper}-\ref{as:b_pol}. Let $\{\theta_i\}_{i=0}^{N-1}$ be the policy parameters generated by DRM-OffP-LR using \eqref{eq:approx_theta_update_H_lr}, and let $\theta_R$ be chosen uniformly at random from this set. Then,
%Set $\alpha=\frac{1}{\sqrt{N}}$, and $m=\sqrt{N}$. Then,
    \begin{align}
    \label{eq:tm_drm_offP_lr}
    &\E\left[\left\lVert \nabla \rho_g(\theta_R)\right\rVert^2\right]
      \leq \frac{2 \left(\rho_g^* - \rho_g(\theta_{0})\right)}{N \alpha}\nonumber\\
&\!+\! 4 M_r^2M_s^2M_e^2M_d^2 \!\left(\!\!\alpha M_{g'}^2L_{\rho'} \!+\! \frac{8(e^2M_{g'}^2 \!+\! M_{g''}^2M_s^2)}{m}\!\!\right)\!\!
\end{align}
%    \begin{align}
%        \label{eq:tm_drm_offP_lr}
%    &\E\left[\left\lVert \nabla \rho_g(\theta_R)\right\rVert^2\right]
%    \leq \frac{2 \!\left(\rho_g^* - \rho_g(\theta_{0})\right)}{\sqrt{N}}\nonumber\\
%    &\quad+ 4 M_r^2M_s^2M_e^2M_d^2 \left(\frac{M_{g'}^2L_{\rho'}}{\sqrt{N}} + \frac{8(e^2M_{g'}^2 + M_{g''}^2M_s^2)}{m}\right),
%\end{align}
%\begin{align*}
%    &\E\left[\left\lVert\nabla\rho_g(\theta_R)\right\rVert^2\right]\leq \frac{2 \left(\rho_g^* - \rho_g(\theta_{0})\right)}{\sqrt{N}} \\
%    &\quad + \frac{4 M_r^2M_s^2M_e^2M_d^2 \left(M_{g'}^2L_{\rho'}+ 8(e^2M_{g'}^2 + M_{g''}^2M_s^2)\right)}{\sqrt{N}},
%\end{align*}
%    In the above, $\rho_g^*=\max_{\theta\in\mathbb{R}^d}\rho_g(\theta)$, $M_r = \frac{r_{\textrm{max}}}{1-\gamma}$, and $L_{\rho'}=2M_r M_e \left( M_h M_{g'}+M_eM_d^2 (M_{g'}+ M_{g''})\right)$. The constants $M_d, M_{g'}, M_{g''}$, and $M_h$ are as defined in \ref{as:nabla_logpi}-\ref{as:g'_bound_2}. The constant $M_s$ is an upper bound on the importance sampling ratio, and $M_e$ is an upper bound on the episode length (see \eqref{eq:M_b} from Lemma \ref{lm:varH} and Lemma \ref{lm:is_ratio} in Section \ref{sec:conv_lr}).
    where $\rho_g^*$, $M_r $, $L_{\rho'}$, $M_d, M_{g'}$, and $M_{g''}$ are as defined in Theorem \ref{tm:drm_onP_lr}. The constant $M_s$ is an upper bound on the importance sampling ratio from \eqref{eq:is_ratio}, and $M_e$ is an upper bound on the episode length from \eqref{eq:M_b}.
\end{theorem}
%\begin{proof}
%    See Section \ref{subsec:offp_lr}.
%\end{proof}
We specialize the bound in \eqref{eq:tm_drm_offP_lr} to a particular choice of step-size $\alpha$, and batch size $m$ in the corollary below.
\begin{corollary}
    \label{cr:drm_offP_lr}
     Set $\alpha=\frac{1}{\sqrt{N}}$, and $m=\sqrt{N}$. Then, under the conditions of Theorem \ref{tm:drm_offP_lr}, we have
\begin{align*}
    &\E\left[\left\lVert\nabla\rho_g(\theta_R)\right\rVert^2\right]\leq \frac{2 \left(\rho_g^* - \rho_g(\theta_{0})\right)}{\sqrt{N}} \\
    &\quad + \frac{4 M_r^2M_s^2M_e^2M_d^2 \left(M_{g'}^2L_{\rho'}+ 8(e^2M_{g'}^2 + M_{g''}^2M_s^2)\right)}{\sqrt{N}}.
\end{align*}
    \end{corollary}
    From Corollary \ref{cr:drm_onP_lr} (or \ref{cr:drm_offP_lr}), we conclude that
%    after $N$ iterations of \eqref{eq:approx_theta_update_G_lr} (or \eqref{eq:approx_theta_update_H_lr}), algorithm DRM-OnP-LR (or DRM-OffP-LR) returns an iterate that satisfies $\E\!\left[\left\lVert\nabla\rho_g(\theta_R)\right\rVert^2\right]\!=\!O\!\left(\nicefrac{1}{\sqrt{N}}\right)$. In other words,
an order $O\left(\nicefrac{1}{\epsilon^2}\right)$ number of iterations are enough to find an $\epsilon$-stationary point for DRM-OnP-LR (or DRM-OffP-LR).
%    Both DRM-OnP-LR and DRM-OffP-LR require $O(\sqrt{N})$ episodes per iteration.
%An observation similar to that in Remark \ref{rm:drm_onP_lr} holds for off-policy RL setting, and the convergence rate in this case is comparable to that in an on-policy RL setting.
%
\section{Convergence proofs}
\label{sec:dpg}
%\subsection{The gradient and the Hessian of the CDF $F_{R^{\theta}}(\cdot)$}
%\label{subsec:cdf}
%In this section, first we prove Lemma \ref{lm:nablaFG}, which express the gradient of the CDF $F_{R^{\theta}}(\cdot)$ as an expectation w.r.t the episodes from the policy $\pi_\theta$. %We replace this expectation by sample average, and forms an estimator for the DRM gradient, see \eqref{eq:nabla_G} and \eqref{eq:hat_nabla_rho_G} in Section \ref{sec:drm}.
%Next, we derive an expression for the Hessian of the CDF. We use this expression to prove the Lipschitz property of the CDF gradient (see Lemma \ref{lm:F_lip}) which in turn is used to prove that the DRM $\rho_g(\theta)$ is a smooth function of $\theta$, see Lemma \ref{lm:nabla_rho_lip}.

%\begin{lemma}
%    \label{lm:nablaFG}
%    $\forall x \in(-M_r,M_r)$,
%    \begin{align*}
    %        &\nabla F_{R^{\theta}}(x) =\E\left[\1\{R^{\theta} \leq x\}\sum\nolimits_{t=0}^{T-1}\nabla\log\pi_\theta(A_t | S_t)\right]. %\textrm{ and}\\
    %%        &\nabla^2 F_{R^{\theta}}(x) =\E\left[\1\{R^{\theta}\leq x\}\left(\sum\nolimits_{t=0}^{T-1}\!\nabla^2\log\pi_\theta(A_t | S_t)\right.\right.\\
    %%        &+ \left.\left.\left[\sum\nolimits_{t=0}^{T-1}\!\nabla\log\pi_\theta(A_t | S_t)\right] \!\!\left[\sum\nolimits_{t=0}^{T-1}\!\nabla\log\pi_\theta(A_t | S_t)\right]^T\right)\right].
    %    \end{align*}
%\end{lemma}
%\subsubsection*{Proof of Lemma \ref{lm:nablaFG}}
\begin{proof}(\textbf{Lemma \ref{lm:nablaFG}})
    Let $\Omega$ denote the set of all sample episodes. For any episode $\omega\in\Omega$, we denote by $T(\omega)$, its length, and $S_t(\omega)$ and $A_t(\omega)$, the state and action at time $t\in\{0,1,\cdots\}$ respectively.
For any $\omega$, let the cumulative discounted reward be
\begin{align*}
    R(\omega)=\sum\nolimits_{t=0}^{T(\omega)-1}\gamma^t r(S_t(\omega),A_t(\omega),S_{t+1}(\omega)).
\end{align*}
Also, let
\begin{align}
    \label{eq:pi_omega}
    \p_\theta(\omega) \!=\!\!\!\!\prod\limits_{t=0}^{T(\omega)-1}\!\!\pi_\theta(A_t(\omega)|S_t(\omega))p(S_{t+1}(\omega),S_t(\omega),A_t(\omega)).
    \end{align}
From \eqref{eq:pi_omega}, we obtain
\begin{align}
    \label{eq:is_grad_onp}
    \frac{\nabla \p_\theta(\omega)}{\p_\theta(\omega)} =\sum\nolimits_{t=0}^{T(\omega)-1}\!\nabla\log\pi_\theta(A_t(\omega) | S_t(\omega)).
    \end{align}
    Now,
    \begin{align}
        &\nabla F_{R^{\theta}}(x)
        \!=\!\nabla\E\!\left[\1\{R^{\theta}\leq x\}\right] \!=\!\nabla \!\sum_{\omega\in\Omega} \1\{R(\omega)\leq x\}\p_\theta(\omega)\nonumber\\
        &\stackrel{(a)}{=}\!\sum_{\omega\in\Omega}\nabla\left( \1\{R(\omega)\!\leq \!x\}\p_\theta(\omega)\right)\nonumber\\%\label{eq:A_FG}\\
        &\stackrel{(b)}{=}\!\sum_{\omega\in\Omega} \1\{R(\omega)\!\leq\! x\}\nabla \p_\theta(\omega)\label{eq:B_FG}\\
        &=\sum_{\omega\in\Omega} \1\{R(\omega)\leq x\} \frac{\nabla \p_\theta(\omega)}{\p_\theta(\omega)}\p_\theta(\omega)\nonumber\\
        &\stackrel{(c)}{=}\sum_{\omega\in\Omega} \1\{R(\omega)\leq x\}\!\sum\limits_{t=0}^{T(\omega)-1}\!\!\!\nabla\log\pi_\theta(A_t(\omega)|S_t(\omega))\p_\theta(\omega)\nonumber\\
        &=\E\left[\1\{R^{\theta}\leq x\}\sum\nolimits_{t=0}^{T-1}\nabla\log\pi_\theta(A_t|S_t)\right].\nonumber
    \end{align}
    In the above, \((a)\) follows by an application of the dominated convergence theorem to interchange the differentiation and the expectation operation. The aforementioned application is allowed since  (i) $\Omega$ is finite and the underlying measure is bounded, as we consider an MDP where the state and actions spaces are finite, and the policies are proper, (ii) $\nabla\log\pi_\theta(A_t|S_t)$ is bounded from \ref{as:nabla_logpi}. The step \((b)\) follows, since for a given episode $\omega$, the cumulative reward $R(\omega)$ does not depend on $\theta$, and \((c)\) follows from \eqref{eq:is_grad_onp}.%\hfill{\qed}
    %    Similarly,
    %    from
    %    \begin{align*}
        %        &\frac{\nabla^2 \p_\theta(\omega)}{\p_\theta(\omega)} =\sum\nolimits_{t=0}^{T(\omega)-1}\nabla^2\log\pi_\theta(A_t(\omega) | S_t(\omega)) + \\
        %        &\left[\!\sum\limits_{t=0}^{T(\omega)-1}\!\!\!\!\!\nabla\log\pi_\theta(A_t(\omega) | S_t(\omega))\!\right]\!\!\!\! \left[\!\sum\limits_{t=0}^{T(\omega)-1}\!\!\!\!\!\nabla\log\pi_\theta(A_t(\omega) | S_t(\omega))\!\right]^\top
        %    \end{align*}
    %    we obtain
    %    \begin{align*}
        %        &\nabla^2 F_{R^{\theta}}(x)
        %        =\E\left[\1\{R^{\theta}\leq x\}\left(\sum\nolimits_{t=0}^{T-1}\!\nabla^2\log\pi_\theta(A_t | S_t)\right.\right.\\&\left.\left.+
        %        \left[\sum\nolimits_{t=0}^{T-1}\!\nabla\log\pi_\theta(A_t | S_t)\right] \!\!\left[\sum\nolimits_{t=0}^{T-1}\!\nabla\log\pi_\theta(A_t | S_t)\right]^T\right)\right].
        %    \end{align*}
\end{proof}

The following lemma establishes an upper bound on the norm of the gradient and the Hessian of the CDF $F_{R^{\theta}}(\cdot)$.
\begin{lemma}
    \label{lm:nablaF_bound}
     $\forall x \in(-M_r,M_r)$,
    \begin{align*}
   \left\lVert \nabla F_{R^{\theta}}(x) \right\rVert \!\leq\! M_e M_d;\,\,
    \left\lVert \nabla^2 F_{R^{\theta}}(x) \right\rVert \!\leq\! M_e M_h +M_e^2M_d^2.
    \end{align*}
\end{lemma}
\begin{proof}
    %    Recall that $T$ denotes the (random) episode length of a proper policy $\pi_\theta$.
    %    From \ref{as:proper}, we infer that $\E[T]<\infty$. This fact in conjunction with $T\geq0$ implies the following bound:
    %    \begin{align}
        %        \label{eq:M_pi}
        %        \exists M_e >0 \textrm{ s.t. } T \leq M_e \textrm{ a.s}.
        %    \end{align}
     Using similar arguments as in proof of Lemma \ref{lm:nablaFG}, we obtain
    \begin{align*}
        &\frac{\nabla^2 \p_\theta(\omega)}{\p_\theta(\omega)} =\sum\nolimits_{t=0}^{T(\omega)-1}\nabla^2\log\pi_\theta(A_t(\omega) | S_t(\omega)) + \\
        &\left[\!\sum\limits_{t=0}^{T(\omega)\!-\!1}\!\!\!\!\!\nabla\!\log\pi_\theta(A_t(\omega) | S_t(\omega)\!)\!\right]\!\!\!\! \left[\!\sum\limits_{t=0}^{T(\omega)\!-\!1}\!\!\!\!\!\nabla\!\log\pi_\theta(A_t(\omega) | S_t(\omega)\!)\!\right]^{\!\!\top}
    \end{align*}
    and which in turn leads to
    \begin{align}
        \label{eq:nablaFG_1}
        &\nabla^2 F_{R^{\theta}}(x)
        =\E\left[\1\{R^{\theta}\leq x\}\left(\sum\nolimits_{t=0}^{T-1}\!\nabla^2\log\pi_\theta(A_t | S_t)\right.\right.\nonumber\\
        &\left.\left.\!+\!
        \left[\sum\limits_{t=0}^{T-1}\!\nabla\log\pi_\theta(A_t | S_t)\right] \!\!\!\left[\sum\limits_{t=0}^{T-1}\!\nabla\log\pi_\theta(A_t | S_t)\right]^T\right)\!\!\right]\!
    \end{align}
    From \ref{as:nabla_logpi}, \ref{as:nabla_logpi_2}, and \eqref{eq:M_pi}, $\forall x\in(-M_r,M_r)$, we obtain
    \begin{align}
        \label{eq:nabla_G_bound}
        \left\lVert\1\{R^{\theta}\leq x\}\sum\limits_{t=0}^{T-1}\nabla\log \pi_{\theta}(A_t|S_t)\right\rVert \leq  M_e M_d \textrm{ a.s.},
    \end{align}
    and
    \begin{align}
        \label{eq:nabla_G_bound1}
        &\left\lVert\1\{R^{\theta}\leq x\}\left(\sum\nolimits_{t=0}^{T-1}\nabla^2\log\pi_\theta(A_t | S_t)\right.\right.\nonumber\\
        &+\left.\left.\left[\sum\limits_{t=0}^{T-1}\nabla\log\pi_\theta(A_t | S_t)\right]\left[\sum\limits_{t=0}^{T-1}\nabla\log\pi_\theta(A_t | S_t)\right]^{\top}\right)\right\rVert\nonumber\\
        &\leq  M_e M_h +M_e^2M_d^2\textrm{ a.s}.
    \end{align}
    From Lemma \ref{lm:nablaFG} and \eqref{eq:nablaFG_1},  $\forall x\in(-M_r,M_r)$, we obtain
    \begin{align*}
        \left\lVert \nabla F_{R^{\theta}}(x) \right\rVert
        \!\leq\! \E\!\left\lVert \1\!\{R^{\theta}\!\leq\! x\}\!\! \sum\limits_{t=0}^{T-1}\!\nabla\log \pi_{\theta}(A_t| S_t)\!\right\rVert
        \!\leq\! M_e M_d,% \label{eq:nabla_F_R_bound}
    \end{align*}
    \begin{align*}
        &\left\lVert \nabla^2 F_{R^{\theta}}(x) \right\rVert
        \leq \E\left\lVert\1\{R^{\theta}\leq x\}\left(\sum\limits_{t=0}^{T-1}\nabla^2\log\pi_\theta(A_t | S_t)\right.\right.\nonumber\\
        &\left.\left.+\left[\sum\limits_{t=0}^{T-1}\nabla\log\pi_\theta(A_t | S_t)\right] \left[\sum\limits_{t=0}^{T-1}\nabla\log\pi_\theta(A_t | S_t)\right]^T\right)\right\rVert\nonumber\\
        &\leq M_e M_h +M_e^2M_d^2, %\label{eq:nabla2_F_R_bound}
    \end{align*}
    where these inequalities follow from \eqref{eq:nabla_G_bound}, \eqref{eq:nabla_G_bound1}, and the assumption that the state and action spaces are finite.
\end{proof}
%%%%%%%%%%%%%%%%%%%%%%%%%%%%%%%%%%%%%%%%%%%%%%%%%%%%%%%%%%%%%%%%%%%%%

%The following lemma establishes the Lipschitzness of the CDF and its gradient.
%\begin{lemma}
%    \label{lm:F_lip}
%    $\forall x \!\in\!(-M_r,M_r)$,
%    \begin{align*}
%        &\left\lvert F_{R^{\theta_1}}(x) - F_{R^{\theta_2}}(x) \right\rvert \leq M_eM_d \left\lVert \theta_1 - \theta_2 \right\rVert, \textrm{and}\\
%        &\left\lVert \nabla F_{R^{\theta_1}}(x) - \nabla F_{R^{\theta_2}}(x) \right\rVert \leq (M_eM_h+M_e^2M_d^2) \left\lVert \theta_1 - \theta_2 \right\rVert.
%    \end{align*}
%\end{lemma}
%\begin{proof}
%    The result follows from Lemma \ref{lm:nablaF_bound} and \cite[Lemma~1.2.2]{nesterov_book}.
%\end{proof}
%%%%%%%%%%%%%%%%%%%%%%%
%%%
%\subsection{DRM policy gradient theorem}
%\label{subsec:dpg}
%In this section, we prove Theorem \ref{thm:nabla_rho_g}, which is the DRM analogue to the policy gradient theorem.
%\subsubsection*{Proof of Theorem \ref{thm:nabla_rho_g}}
\begin{proof}(\textbf{Theorem \ref{thm:nabla_rho_g}})
    %    Recall that $T$ denotes the (random) episode length of a proper policy $\pi_\theta$.
    %    From \ref{as:proper}, we infer that $\E[T]<\infty$. This fact in conjunction with $T\geq0$ implies the following bound:
    %    \begin{align}
        %        \label{eq:M_pi}
        %        \exists M_e >0 \textrm{ s.t. } T \leq M_e \textrm{ a.s}.
        %    \end{align}
    %    From \ref{as:nabla_logpi} and \eqref{eq:M_pi}, for any $x\in[-M_r,M_r]$, we have
    %    \begin{align}
        %        \label{eq:nabla_G_bound}
        %        \left\lVert\1\{R^{\theta}\!\leq \!x\}\!\sum_{t=0}^{T-1}\!\nabla\log \pi_{\theta}(A_t\mid S_t)\right\rVert \!\leq\!  M_e M_l \textrm{ a.s}.
        %    \end{align}
    %    From \eqref{eq:nabla_F_R}, for any $x\in[-M_r,M_r]$, we have
    %    \begin{align}
        %     \left\lVert \nabla F_{R^{\theta}}(x) \right\rVert
        %     &\leq \E\left[\left\lVert \1\{R^{\theta}\leq x\} \sum_{t=0}^{T-1} \nabla\log \pi_{\theta}(A_t| S_t)\right\rVert\right]\nonumber\\
        %     &\leq M_e M_l, \label{eq:nabla_F_R_bound}
        %     \end{align}
    % where the last inequality follows from \eqref{eq:nabla_G_bound}, and the assumption that the state and action spaces are finite.
    Notice that,
    \begin{align*}
        &\nabla \rho_g(\theta)\\
        &=\nabla\!\!\int_{-M_r}^{0}\!\!\!\left(g(1\!-\!F_{R^{\theta}}(x))\!-\!1\right)dx  +\! \nabla\!\!\int_{0}^{M_r}\!\!\!\! g(1\!-\!F_{R^{\theta}}(x))dx \\
        &\stackrel{(a)}{=}\!\!\!\int_{-M_r}^{0}\!\!\!\!\nabla \left(g(1\!-\!F_{R^{\theta}}(x))\!-\!1\right) dx + \!\!\int_{0}^{M_r}\!\!\nabla g(1\!-\!F_{R^{\theta}}(x)) dx \\%\label{eq:B}\\
        &=-\int_{-M_r}^{M_r} g'(1-F_{R^{\theta}}(x)) \nabla F_{R^{\theta}}(x)dx.
    \end{align*}
    In the above, \((a)\) follows by an application of the dominated convergence theorem to interchange the differentiation and the integral operation. The aforementioned application is allowed since
    (i) $\rho_g(\theta)$ is finite for any $\theta \!\in\! \R^d$; (ii)  $\left\lvert g'(\cdot)\right\rvert \!\leq\! M_{g'}$ from \ref{as:g'_bound}, and $\nabla F_{R^{\theta}}(\cdot)$ is bounded from Lemma \ref{lm:nablaF_bound}. The bounds on $g'$ and $\nabla F_{R^{\theta}}$ imply\\$\int_{-M_r}^{M_r} \left\lVert g'(1-F_{R^{\theta}}(x)) \nabla F_{R^{\theta}}(x)\right\rVert dx \leq 2 M_rM_{g'} M_eM_d$.
    %        the limits of the integration are finite, (ii) $g(\cdot)$ is bounded by definition, and
\end{proof}
\begin{proof}(\textbf{Lemma \ref{lm:nabla_rho_lip}})
     Using the mean value theorem, we obtain
    $g'(t)-g'(t') = g''(\tilde{t}) (t - t'),\tilde{t} \in (t,t')$,
    and from \ref{as:g'_bound}, we have $\left\lvert g''(\tilde{t}) \right\rvert \leq M_{g''}, \forall \tilde{t} \in(0,1)$. Hence,
    \begin{align}
        \label{eq:g_lip}
        \left\lvert g'(t) -g'(t')\right\rvert \leq M_{g''} \left\lvert t - t'\right\rvert\;\forall t,t'\in(0,1).
       \end{align}
    From Lemma \ref{lm:nablaF_bound} and \cite[Lemma~1.2.2]{nesterov_book}, for any $ x \in (-M_r,M_r)$ and $\forall \theta_1, \theta_2 \in \R^d$, we obtain
        \begin{align}
            &\left\lvert F_{R^{\theta_1}}(x) - F_{R^{\theta_2}}(x) \right\rvert \leq M_eM_d \left\lVert \theta_1 - \theta_2 \right\rVert,  \label{eq:F_lip}\\
            &\left\lVert \nabla F_{R^{\theta_1}}(x) \!-\! \nabla F_{R^{\theta_2}}(x) \right\rVert \!\leq\! (M_eM_h+M_e^2M_d^2) \left\lVert \theta_1 \!-\! \theta_2 \right\rVert.\label{eq:F_smooth}
        \end{align}
    From Theorem \ref{thm:nabla_rho_g}, we obtain
    \begin{align*}
        &\left\lVert \nabla\rho_g(\theta_1) \!-\! \nabla \rho_g(\theta_2) \right\rVert
        \leq \!\int_{-M_r}^{M_r}\!\!\!\left\lVert  g'(1\!-\!F_{R^{\theta_1}}(x)) \nabla F_{R^{\theta_1}}(x)\right.\\
        &\quad\left.- g'(1-F_{R^{\theta_2}}(x)) \nabla F_{R^{\theta_2}}(x) \right\rVert dx\\
        &\leq \int_{-M_r}^{M_r}\left\lVert  g'(1-F_{R^{\theta_1}}(x)) \nabla F_{R^{\theta_1}}(x)\right.\\
        &-g'(1\!-\!F_{R^{\theta_1}}(x)) \nabla F_{R^{\theta_2}}(x) + g'(1\!-\!F_{R^{\theta_1}}(x)) \nabla F_{R^{\theta_2}}(x) \\
        &\quad \left.-g'(1-F_{R^{\theta_2}}(x)) \nabla F_{R^{\theta_2}}(x) \right\rVert dx   \\
        &\leq \int_{-M_r}^{M_r}  \left(\left\lvert g'(1-F_{R^{\theta_1}}(x)) \right\rvert  \left\lVert \nabla F_{R^{\theta_1}}(x) -  \nabla F_{R^{\theta_2}}(x)\right\rVert\right.\\
        &+\left.\left\lVert\nabla F_{R^{\theta_2}}(x)\right\rVert \left\lvert  g'(1-F_{R^{\theta_1}}(x)) - g'(1-F_{R^{\theta_2}}(x))\right\rvert \right)dx\\
        %        &\leq \int_{-M_r}^{M_r} M_{g'} \left\lVert \nabla F_{R^{\theta_1}}(x) -  \nabla F_{R^{\theta_2}}(x)\right\rVert \\
        %        &\hfill+M_eM_d \left\lvert  g'(1-F_{R^{\theta_1}}(x))  - g'(1-F_{R^{\theta_2}}(x))\right\rvert dx  \\
        %        &\hfill \text{(from lemmas \ref{lm:g_bound} and \ref{lm:nablaF_bound})} \\
        &\stackrel{(a)}{\leq} \int_{-M_r}^{M_r}\left(M_{g'} \left\lVert \nabla F_{R^{\theta_1}}(x) -  \nabla F_{R^{\theta_2}}(x)\right\rVert\right.\\
        &\quad \left.+M_eM_d M_{g''}\left\lvert F_{R^{\theta_1}}(x)  - F_{R^{\theta_2}}(x)\right\rvert \right)dx \\%\tag*{(from \ref{as:g'_bound} and Lemmas \ref{lm:g_lip}, \ref{lm:nablaF_bound})}\\
        &\stackrel{(b)}{\leq} \int_{-M_r}^{M_r} \left(M_{g'} (M_eM_h+M_e^2M_d^2)\left\lVert  \theta_1  - \theta_2\right\rVert\right.\\
        &\quad\left.+ M_e^2M_d^2 M_{g''}\left\lVert  \theta_1  - \theta_2\right\rVert\right) dx \\%\tag*{(from Lemma \ref{lm:F_lip})}\\
        &\leq 2M_r M_e \left( M_h M_{g'}+M_eM_d^2 (M_{g'}+ M_{g''})\right) \left\lVert  \theta_1 - \theta_2\right\rVert\\
        &=L_{\rho'}\left\lVert  \theta_1 - \theta_2\right\rVert,
    \end{align*}
    where \((a)\) follows from \ref{as:g'_bound}, \eqref{eq:g_lip}, and Lemma \ref{lm:nablaF_bound}, and \((b)\) follows from \eqref{eq:F_lip}-\eqref{eq:F_smooth}.%\hfill{\qed}
\end{proof}

In the lemma below, we establish an upper bound on the variance of the DRM gradient estimate $\widehat{\nabla}^G \rho_g(\theta)$ as defined in \eqref{eq:hat_nabla_rho_G}. Subsequently, we use this result to prove Lemma \ref{lm:biasG}, and Theorem \ref{tm:drm_onP_lr}.
\begin{lemma}
    \label{lm:varG}
    $\E\left\lVert \widehat{\nabla}^G \rho_g(\theta) \right\rVert^2 \leq 4 M_r^2 M_{g'}^2 M_e^2M_d^2.$
\end{lemma}
\begin{proof}
%    From \ref{as:proper}, it is easy to see that the episode length $T$ is bounded for any episode. So,
%    \begin{align}
%        \label{eq:M_pi}
%        \exists M_e >0 \textrm{ s.t. } T \leq M_e \textrm{ a.s}.
%    \end{align}
    From \eqref{eq:nabla_G}, using \eqref{eq:nabla_G_bound} from Lemma \ref{lm:nablaF_bound}, we obtain
    \begin{align}
        \label{eq:nabla_Gm_bound}
        \left\lVert\widehat{\nabla}G^m_{R^{\theta}}(x)\right\rVert^2 \leq M_e^2M_d^2 \textrm{ a.s.},\forall x\in(-M_r,M_r).
    \end{align}
    From \eqref{eq:hat_nabla_rho_G}, and  \ref{as:g'_bound}, we obtain
    \begin{align*}
        &\E\left\lVert\widehat{\nabla}^G \rho_g(\theta) \right\rVert^2
       % \!=\! \E\!\left[\!\left\lVert\int\limits_{-M_r}^{M_r}\!\! g'\left(1\!-\!G^m_{R^{\theta}}(x)\right) \!\widehat{\nabla}G^m_{R^{\theta}}(x)dx \right\rVert^2\right]\\
        \leq M_{g'}^2 \E\left\lVert \int_{-M_r}^{M_r} \widehat{\nabla}G^m_{R^{\theta}}(x)dx \right\rVert^2\\%\tag*{(from Lemma \ref{lm:g_bound})}\\
        &\stackrel{(a)}{\leq} 2 M_r M_{g'}^2 \E\left[\int_{-M_r}^{M_r} \left\lVert\widehat{\nabla}G^m_{R^{\theta}}(x)\right\rVert^2dx \right]\\%\tag*{(from Cauchy-Schwarz inequality)}\\
        &\stackrel{(b)}{\leq} 2 M_r M_{g'}^2\int_{-M_r}^{M_r}\E\left\lVert\widehat{\nabla}G^m_{R^{\theta}}(x)\right\rVert^2 dx\\% \tag*{(from Fubini's theorem)}\\
        &\stackrel{(c)}{\leq} 2 M_r M_{g'}^2\int_{-M_r}^{M_r}  M_e^2M_d^2\, dx = 4 M_r^2 M_{g'}^2M_e^2M_d^2,
    \end{align*}
    where \((a)\) follows from the Cauchy-Schwarz inequality, \((b)\) follows from the Fubini's theorem, and \((c)\) follows from \eqref{eq:nabla_Gm_bound}.
\end{proof}
%%%%%%%%%%%%%%%%%%%%%%%%%%%%%%%%%%%%%%%%%%%%%%%%%%%%%%%%%%%%%%%%%%%%
%Now we prove Lemma \ref{lm:biasG} which bound the the mean squared error of the DRM gradient estimate $\widehat{\nabla}^G \rho_g(\theta)$.
%\begin{lemma}
%\label{lm:biasG}
%\begin{align*}
%\E\!\left[\!\left\lVert\widehat{\nabla}^G\! \rho_g(\theta) \!-\!\nabla\!\rho_g(\theta)\right\rVert^2\right] \!\!\leq\! \frac{32 M_r^2M_e^2M_d^2(e^2M_{g'}^2 \!+\! M_{g''}^2)}{m}
%\end{align*}
%\end{lemma}
%\subsubsection*{Proof of Lemma \ref{lm:biasG}}
\begin{proof}(\textbf{Lemma \ref{lm:biasG}})
 From the fact that\\$\forall x\in(-M_r,M_r)$, $\left\lvert\1\{R^{\theta}\leq x\}\right\rvert \leq  1$ a.s.,
we observe that $\left\{m'\left(G^{m'}_{R^{\theta}}(x) - F_{R^{\theta}}(x)\right)\right\}_{m'=1}^{m}$ is a set of partial sums of bounded mean zero r.v.s, and hence they are martingales.
Using Azuma-Hoeffding's inequality, we obtain $\forall x\in(-M_r,M_r)$,
\begin{align}
    &\p\left(\left\lvert G^m_{R^{\theta}}(x) - F_{R^{\theta}}(x) \right\rvert > \epsilon\right) \leq 2\exp\left(-\frac{m\epsilon^2}{2}\right), \textrm{ and} \nonumber\\
    &\E\left\lvert G^m_{R^{\theta}}(x) \!-\! F_{R^{\theta}}(x)\right\rvert^2
    \!=\!\int\limits_{0}^{\infty}\! \p\left(\left\lvert G^m_{R^{\theta}}(x) \!-\! F_{R^{\theta}}(x)\right\rvert \!>\! \sqrt{\epsilon}\right)d\epsilon \nonumber\\
    &\leq \int_{0}^{\infty} 2\exp\left(-\frac{m\epsilon}{2}\right) d\epsilon = \frac{4 }{m}.\label{eq:G_err}
\end{align}
%From \ref{as:nabla_logpi} and \eqref{eq:M_pi}, we obtain
%\begin{align}
%\label{eq:nabla_G_bound}
%\left\lVert\1\{R^{\theta}\!\leq \!x\}\!\sum_{t=0}^{T-1}\!\nabla\log \pi_{\theta}(A_t\mid S_t)\right\rVert \!\leq\!  M_e M_l \textrm{ a.s.}, \forall x.
%\end{align}
Using \eqref{eq:nabla_G_bound} from Lemma \ref{lm:nablaF_bound}, we observe that\\$\left\{m'\left(\widehat{\nabla}G^{m'}_{R^{\theta}}(x) \!-\! \nabla F_{R^{\theta}}(x) \right)\right\}_{m'=1}^{m}$ is a set of partial sums of bounded mean zero r.v.s, and hence they are martingales. Using vector version of Azuma-Hoeffding inequality from Theorem~1.8-1.9 in \cite{hayes2005large}, for any $x\in(-M_r,M_r)$, we have
\begin{align}
%\label{eq:nabla_G_prob}
&\p\!\left(\!\left\lVert \widehat{\nabla}G^m_{R^{\theta}}(x) \!-\! \nabla F_{R^{\theta}}(x) \right\rVert \!>\! \epsilon\!\right) \!\leq\! 2e^2\!\exp\left(\frac{-m\epsilon^2}{2M_e^2M_d^2}\!\right),\nonumber\\
%\end{align*}
%From \eqref{eq:nabla_G_prob}, we obtain
%and
%\begin{align}
&\E\left\lVert \nabla F_{R^{\theta}}(x) - \widehat{\nabla}G^m_{R^{\theta}}(x) \right\rVert^2 \nonumber\\
&= \int_{0}^{\infty} \p\left( \left\lVert \nabla F_{R^{\theta}}(x) - \widehat{\nabla}G^m_{R^{\theta}}(x) \right\rVert > \sqrt{\epsilon}\right)d\epsilon \nonumber\\
&\leq \int_{0}^{\infty} 2e^2\exp\left(-\frac{m\epsilon}{2M_e^2M_d^2}\right) d\epsilon = \frac{4e^2M_e^2M_d^2}{m}.\label{eq:nabla_G_err}
\end{align}
From \eqref{eq:nabla_Gm_bound}, for any $x\in(-M_r,M_r)$, we have
\begin{align}
    &\E\left\lVert\left(g'(1-F_{R^{\theta}}(x))-  g'(1-G^m_{R^{\theta}}(x))\right) \widehat{\nabla}G^m_{R^{\theta}}(x) \right\rVert^2\nonumber\\
    &\leq M_e^2M_d^2\,\E\left\lvert g'(1-F_{R^{\theta}}(x)) - g'(1-G^m_{R^{\theta}}(x))\right\rvert^2\nonumber\\
    &\leq M_{g''}^2M_e^2M_d^2 \,\E\left\lvert G^m_{R^{\theta}}(x)-F_{R^{\theta}}(x)\right\rvert^2
    \leq \frac{4M_{g''}^2M_e^2M_d^2}{m},\label{eq:G_A}
\end{align}
where the last two inequalities follow from \eqref{eq:g_lip} and \eqref{eq:G_err}. From  \ref{as:g'_bound}, for any $x\in(-M_r,M_r)$, we have
\begin{align}
\label{eq:nablaG_A}
&\E\left\lVert g'(1-F_{R^{\theta}}(x)) \left(\nabla F_{R^{\theta}}(x)-\widehat{\nabla}G^m_{R^{\theta}}(x)\right)\right\rVert^2\nonumber\\
&\leq M_{g'}^2\,\E\left\lVert \nabla F_{R^{\theta}}(x) - \widehat{\nabla}G^m_{R^{\theta}}(x) \right\rVert^2 \leq\frac{4e^2M_{g'}^2M_e^2M_d^2}{m},
\end{align}
where the last inequality follows  from \eqref{eq:nabla_G_err}.

From \eqref{eq:nabla_rho_g_1}, \eqref{eq:hat_nabla_rho_G}, and Cauchy-Schwarz inequality, we have
\begin{align}
&\E\left\lVert\widehat{\nabla}^G \rho_g(\theta) -\nabla\rho_g(\theta)\right\rVert^2 \nonumber\\
%&= \E\left[\left\lVert \int\limits_{-M_r}^{M_r} \left(g'(1-F_{R^{\theta}}(x)) \nabla F_{R^{\theta}}(x)\right.\right. \right.\nonumber\\
%&\quad \left.\left. \left.-  g'(1-G^m_{R^{\theta}}(x)) \widehat{\nabla}G^m_{R^{\theta}}(x)\right)dx \right\rVert^2\right]\nonumber\\
&\leq 2M_r \E\left[\int_{-M_r}^{M_r} \left\lVert g'(1-F_{R^{\theta}}(x)) \nabla F_{R^{\theta}}(x)\right.\right. \nonumber\\
&\quad \left.\left. -  g'(1-G^m_{R^{\theta}}(x)) \widehat{\nabla}G^m_{R^{\theta}}(x)\right\rVert^2 dx\right]\nonumber\\
&\stackrel{(a)}{\leq} 2M_r \int_{-M_r}^{M_r}\E\left[\left\lVert g'(1-F_{R^{\theta}}(x)) \nabla F_{R^{\theta}}(x)\right.\right. \nonumber\\
&\quad \left.\left. -  g'(1-G^m_{R^{\theta}}(x)) \widehat{\nabla}G^m_{R^{\theta}}(x)\right\rVert^2 \right]dx\nonumber\\%\tag*{(from Fubini's theorem)}\\
&\leq 2M_r \int_{-M_r}^{M_r}\!\!\E\left[\left\lVert g'(1\!-\!F_{R^{\theta}}(x)) \left(\nabla F_{R^{\theta}}(x) \!-\! \widehat{\nabla}G^m_{R^{\theta}}(x)\right)\right.\right.\nonumber\\
&\left.\left.+ \left(g'(1-F_{R^{\theta}}(x))-  g'(1-G^m_{R^{\theta}}(x))\right) \widehat{\nabla}G^m_{R^{\theta}}(x) \right\rVert^2\right]dx\nonumber\\
&\stackrel{(b)}{\leq}\! 4M_r \!\!\!\int\limits_{-M_r}^{M_r} \!\!\!\left(\E\left\lVert g'(1\!-\!F_{R^{\theta}}(x))\!
\left(\nabla F_{R^{\theta}}(x)\!-\!\widehat{\nabla}G^m_{R^{\theta}}(x)\right)\right\rVert^2\right.\nonumber\\
& \left.+ \E\left\lVert\left(g'(1\!-\!F_{R^{\theta}}(x))\!-\!  g'(1\!-\!G^m_{R^{\theta}}(x))\right)\!
\widehat{\nabla}G^m_{R^{\theta}}(x) \right\rVert^2\right)dx \nonumber\\%\tag*{(since $\lVert a\!+\!b \rVert^2 \leq 2\lVert a \rVert^2\!+\! 2\Vert b \rVert^2$)}\\
&\stackrel{(c)}{\leq} \frac{32 M_r^2M_e^2M_d^2(e^2M_{g'}^2 + M_{g''}^2)}{m},\label{eq:biasG}
\end{align}
where \((a)\) follows from the Fubini's theorem, \((b)\) follows from the fact that $\lVert x+y \rVert^2 \leq 2\lVert x \rVert^2+ 2\Vert y \rVert^2$, and \((c)\) follows from \eqref{eq:G_A} and \eqref{eq:nablaG_A}.%\hfill{\qed}
\end{proof}
%%%%%%%%%%%%%%%%%%%%%%%%%%%%%%%%%%%%%%%%%%%%%%%%%%%%%%%%%%%%%%%%%%%%%%%%%%%%%
%\subsubsection{Non-asymptotic bound for DRM-OnP-LR}
%Now, we prove Theorem \ref{tm:drm_onP_lr} that presents a non-asymptotic bound for DRM-OnP-LR.
%\subsubsection*{Proof of Theorem \ref{tm:drm_onP_lr}}
\begin{proof}(\textbf{Theorem \ref{tm:drm_onP_lr}})
    Using the fundamental theorem of calculus, we obtain
    \begin{align}
        & \rho_g(\theta_k) - \rho_g(\theta_{k+1})
        =\langle \nabla \rho_g(\theta_k), \theta_k - \theta_{k+1} \rangle \nonumber\\
        &+\! \int_0^1\!\!\! \left\langle  \nabla \rho_g(\theta_{k+1}+\tau(\theta_k\!-\!\theta_{k+1}))\!-\!\nabla \rho_g(\theta_k), \theta_k \!-\! \theta_{k+1} \right\rangle d\tau\nonumber\\
        &\leq\langle \nabla \rho_g(\theta_k), \theta_k - \theta_{k+1} \rangle \nonumber\\
        &+\!\!\!\int_0^1 \!\!\!\!\left\lVert\nabla\rho_g(\theta_{k+1}\!+\!\tau(\theta_k \!-\! \theta_{k+1}))\!-\!\nabla\rho_g(\theta_k)\right\rVert \!\left\lVert \theta_k \!-\! \theta_{k+1} \right\rVert\! d\tau\nonumber\\
        &\stackrel{(a)}{\leq} \left \langle \nabla \rho_g(\theta_k), \theta_k \!-\! \theta_{k+1} \right \rangle
        + L_{\rho'}\left\lVert \theta_k \!-\! \theta_{k+1} \right\rVert^2  \int_0^1 (1\!-\!\tau) d\tau \nonumber\\%\tag*{(from Lemma \ref{lm:nabla_rho_lip})}\\
        &= \left \langle \nabla \rho_g(\theta_k), \theta_k - \theta_{k+1} \right \rangle + \frac{L_{\rho'}}{2}\left\lVert \theta_k - \theta_{k+1} \right\rVert^2 \nonumber\\
        &= \alpha\left\langle\nabla \rho_g(\theta_k), -\widehat{\nabla}^G\!\rho_g(\theta_k) \right \rangle
        + \frac{L_{\rho'}\alpha^2}{2} \left\lVert\widehat{\nabla}^G\rho_g(\theta_k)\right\rVert^2\nonumber\\
       &= \alpha\left \langle \nabla \rho_g(\theta_k), \nabla \rho_g(\theta_k)-\widehat{\nabla}^G \rho_g(\theta_k) \right \rangle \nonumber\\
       &\quad-\alpha\left \lVert \nabla \rho_g(\theta_k)\right \rVert^2
       + \frac{L_{\rho'}\alpha^2}{2}\left\lVert \widehat{\nabla}^G\!\rho_g(\theta_k) \right\rVert^2 \nonumber\\
       &\stackrel{(b)}{\leq} \frac{\alpha}{2}\left \lVert \nabla \rho_g(\theta_k) \right \rVert^2 + \frac{\alpha}{2}\left \lVert \nabla \rho_g(\theta_k) -\widehat{\nabla}^G\!\rho_g(\theta_k) \right \rVert^2 \nonumber\\
       &\quad-\alpha\left \lVert \nabla \rho_g(\theta_k)\right \rVert^2
       + \frac{L_{\rho'}\alpha^2}{2}\left\lVert \widehat{\nabla}^G \rho_g(\theta_k) \right\rVert^2\nonumber\\%\tag*{(since $2\langle a, b \rangle \leq \lVert a \rVert^2+ \Vert b \rVert^2$)}\\
       &= \frac{\alpha}{2}\left \lVert \nabla \rho_g(\theta_k) -\widehat{\nabla}^G\!\rho_g(\theta_k) \right \rVert^2 \nonumber\\
       &\quad-\frac{\alpha}{2}\left \lVert \nabla \rho_g(\theta_k)\right \rVert^2
       + \frac{L_{\rho'}\alpha^2}{2}\left\lVert \widehat{\nabla}^G \rho_g(\theta_k) \right\rVert^2,\label{eq:1_lr_onp}
    \end{align}
where \((a)\) follows from Lemma \ref{lm:nabla_rho_lip}, and \((b)\) follows from the fact that $2\langle x, y \rangle \leq \lVert x \rVert^2+ \Vert y \rVert^2$.

Taking expectations on both sides of \eqref{eq:1_lr_onp}, we obtain
\begin{align}
&\alpha\E\left \lVert \nabla \rho_g(\theta_k)\right \rVert^2%\nonumber\\
\leq  2\E\left[\rho_g(\theta_{k+1}) - \rho_g(\theta_{k})\right] \\
&+ L_{\rho'}\alpha^2 \E\left\lVert \widehat{\nabla}^G \rho_g(\theta_k) \right\rVert^2%\nonumber\\
+  \alpha \E\left \lVert \nabla \rho_g(\theta_k) \!-\!\widehat{\nabla}^G \rho_g(\theta_k) \right \rVert^2 \nonumber\\
&\leq  2\E\left[\rho_g(\theta_{k+1}) - \rho_g(\theta_{k})\right] \nonumber\\
&+  \alpha 4 M_r^2M_e^2M_d^2 \!\left(\alpha M_{g'}^2L_{\rho'} + \frac{8}{m}(e^2M_{g'}^2 \!+\! M_{g''}^2)\!\right),\label{eq:2_lr_onp}
\end{align}
where the last inequality follows from Lemmas \ref{lm:biasG} and \ref{lm:varG}.
Summing up \eqref{eq:2_lr_onp} from $k=0,\cdots,N-1$, we obtain
\begin{align*}
&\alpha\sum\nolimits_{k=0}^{N-1}\E\left \lVert \nabla \rho_g(\theta_k)\right \rVert^2\leq  2 \E\left[\rho_g(\theta_{N}) \!- \!\rho_g(\theta_{0})\right] \\
&\quad+ N \alpha 4 M_r^2M_e^2M_d^2 \left(\alpha M_{g'}^2L_{\rho'} + \frac{8(e^2M_{g'}^2 + M_{g''}^2)}{m}\right).
\end{align*}
Since $\theta_R$ is chosen uniformly at random from the policy iterates $\{\theta_0,\cdots,\theta_{N-1}\}$, we obtain
\begin{align*}
&\E\left\lVert \nabla \rho_g(\theta_R)\right\rVert^2
\!=\! \frac{1}{N}\sum\limits_{k=0}^{N-1}\E\left\lVert \nabla \rho_g(\theta_k)\right \rVert^2
\!\leq\! \frac{2 \left(\rho_g^* - \rho_g(\theta_{0})\right)}{N \alpha} \\
&\quad + 4 M_r^2M_e^2M_d^2 \left(\alpha M_{g'}^2L_{\rho'} + \frac{8(e^2M_{g,}^2 + M_{g''}^2)}{m}\right).
%&\stackrel{(a)}{\leq} \frac{2 \left(\rho_g^* - \rho_g(\theta_{0})\right)}{\sqrt{N}}\\
%&\quad+ 4 M_r^2M_e^2M_d^2 \left(\frac{M_{g'}^2L_{\rho'}}{\sqrt{N}} + \frac{8(e^2M_{g'}^2 + M_{g''}^2)}{m}\right),
%&\stackrel{(a)}{\leq} \frac{2 \left(\rho_g^* - \rho_g(\theta_{0})\right)}{\sqrt{N}}\\
%&\quad+ \frac{4 M_r^2M_e^2M_d^2 \left(M_{g'}^2L_{\rho'} + 8(e^2M_{g'}^2 + M_{g''}^2)\right)}{\sqrt{N}},
\end{align*}
%\hfill{\qed}
%where \((a)\) follows from $\alpha=\frac{1}{\sqrt{N}}$.\hfill{\qed}
%where \((a)\) follows from $\alpha=\frac{1}{\sqrt{N}}$, and $m=\sqrt{N}$.\hfill{\qed}
\end{proof}
%%%%%%%%%%%%%%%%%%%%%%%%%%%%%%%%%%%%%%%%%%%%%%%%%%%%%%%%%%%%%%%%%%%%%%%%%%%%%
%
\section{Simulation Results}
\label{sec:sim}
We conducted experiments on a control problem known as Frozen Lake, sourced from the OpenAI Gym toolkit \cite{openai}. We customized the environment, depicted in Figure \ref{fg:sim}. The state space comprises a $6\times9$ grid, while the action space consists of $\mathscr{A}=\{left, down, right, up \}$. Each action corresponds to movement in the specified direction. When selecting an action, there is a $0.9$ probability that the agent moves in the intended direction, and a $0.05$ probability for each of the adjacent directions. Episodes terminate after $100$ steps, if the agent falls into hole H, or upon reaching the goal G. The rewards are assigned as follows: $+10$ for reaching the goal G, $-10$ for falling into hole H, and $-0.025$ for stepping on a frozen state F.

We performed experiments utilizing our DRM-OnP-LR algorithm, employing various distortion functions, including CVaR and identity function (see Table \ref{tb:g} for the mathematical expressions defining distortion functions). Our algorithm DRM-OnP-LR was executed over $N=10000$ iterations, with a batch size of $m=\sqrt{N}$ and a stepsize of $\alpha=\sqrt{N}$.

The reward plots depicted in Figure \ref{fg:sim} indicate that the DRM with a logarithmic distortion function performs better than other distortion functions, notably outperforming both CVaR and the identity function. Due to the high stochasticity of the grid, there is a notable risk of falling into the hole H when the agent takes the shortest path toward the goal G. Our simulation experiments revealed that when utilizing the logarithmic distortion function, the agent consistently chooses a path that avoids the holes H, resulting in enhanced average rewards.
%\begin{figure}%[t]
%    \caption{Modified Frozen Lake Simulation}
%    \label{fg:sim}
%     \scalebox{0.9}{
%        \centerline{\includegraphics[width=0.7\columnwidth]{plots/frozen_lake.png}}
%    }
%%    \vskip -0.2in
%\end{figure}
\begin{figure*}
    \caption{Modified Frozen Lake}
    \label{fg:sim}
      \begin{subfigure}{0.32\linewidth}
      \caption{Simulation Grid}
          \includegraphics[width=\linewidth]{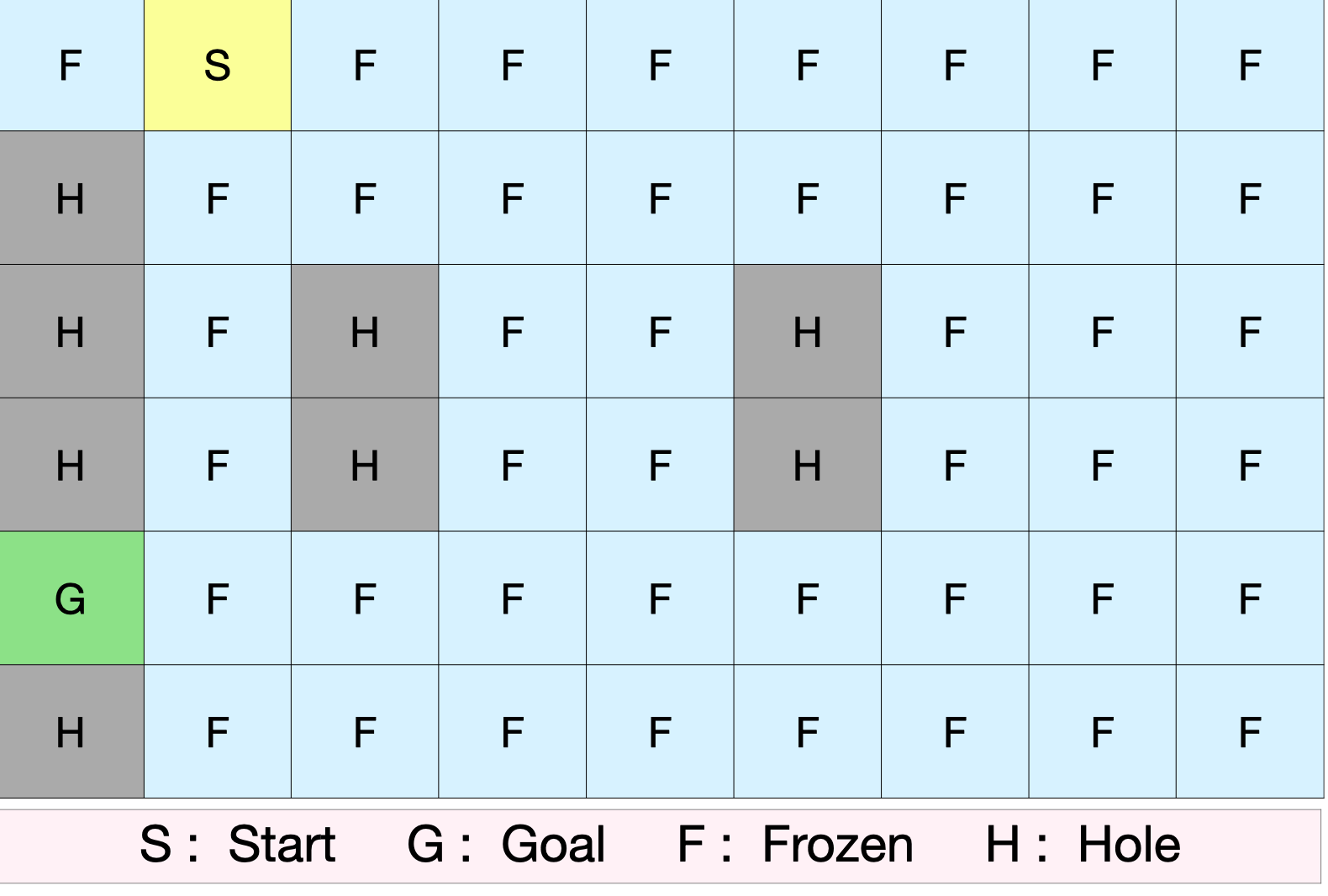}
  \end{subfigure}
  \hfil
    \begin{subfigure}{0.32\linewidth}
        \caption{Training Plots}
            \includegraphics[width=\linewidth]{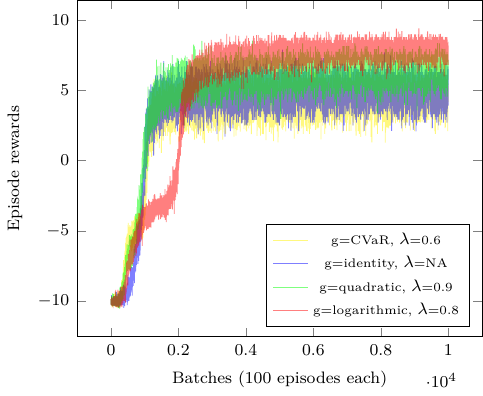}
    \end{subfigure}
    \hfil
    \begin{subfigure}{0.32\linewidth}
        \caption{Testing Plots}
            \includegraphics[width=\linewidth]{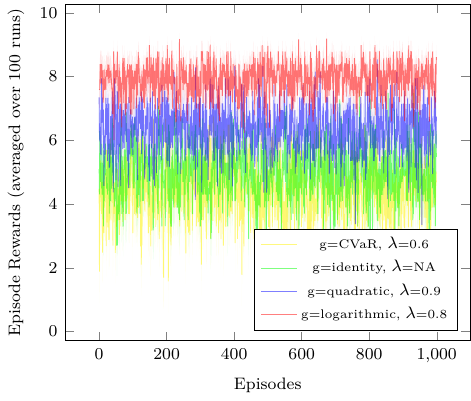}
    \end{subfigure}
\end{figure*}
%\section{Non-asymptotic convergence analysis (SF-based gradient estimation scheme)}
%\label{sec:conv_sf}
%\input{proofs_sf}
%
%\section{Simplifying the estimate of the DRM and its gradient using order statistics}
%\label{sec:conv_aux}
%\input{proofs_aux}
%
\section{Conclusions and future work}
\label{sec:conclusions}
We proposed DRM-based policy gradient algorithms for risk-sensitive RL control. We employed LR-based gradient estimation schemes in on-policy as well as off-policy RL settings, and provided non-asymptotic bounds that establish convergence to an approximate stationary point of the DRM.

As a future work, it would be interesting to study DRM optimization in a risk-sensitive RL setting with feature-based representations, and function approximation. In this setting, one could consider an actor-critic algorithm for DRM optimization, and study its non-asymptotic performance.
\bibliographystyle{plain}
\bibliography{ref}
\appendix
\section{Simplifying the estimate of the DRM gradient using order statistics}
\label{sec:conv_aux}
%%%%%%%%%%%%%%%%%%%%%%%%
%Next, we prove Lemma \ref{lm:hat_nabla_rho_G} and \ref{lm:hat_nabla_rho_H} which simplify the estimate the gradient of the DRM using order statistics in on-policy and off-policy RL settings, respectively.
%\begin{lemma}
%    \label{lm:hat_nabla_rho_G}
%    \begin{align*}
%    &\widehat{\nabla}^G\!\rho_g(\theta)\!=\!\frac{1}{m}\sum\nolimits_{i=1}^{m-1} \!\!\left( R^\theta_{(i)}\!-\!R^\theta_{(i+1)}\right) g'\!\left(1\!-\!\frac{i}{m}\right) \sum\nolimits_{j=1}^i \!\!\nabla l^{\theta}_{(j)}\\
%    &\quad+  \frac{1}{m}\left(R^\theta_{(m)} -M_r\right) g'_{+}(0)\sum\nolimits_{j=1}^m\nabla l^{\theta}_{(j)}.
%    \end{align*}
%\end{lemma}
%\subsection{Proof of Lemma \ref{lm:hat_nabla_rho_G}}
%\label{subsec:pf_hat_nabla_rho_G}
\begin{proof}(\textbf{Lemma \ref{lm:hat_nabla_rho_G}})
%Following the proof of Lemma \ref{lm:hat_rho_G}, we rewrite $G^m_{R^{\theta}}(x)$ from \eqref{eq:G} as \eqref{eq:G2}.
%    \begin{align}
%        \label{eq:G2}
%        G^m_{R^{\theta}}(x) =
%        \begin{cases}
%            0,&\textrm{if } x < R^\theta_{(1)}\\
%            \frac{i}{m},&\textrm{if } R^\theta_{(i)} \leq x < R^\theta_{(i+1)}, i\in\{1,\!\cdots\!,m-1\}\\
%            1,&\textrm{if } x \geq R^\theta_{(m)},
%        \end{cases}
%    \end{align}
%where $R^\theta_{(i)}$ is the $i^{th}$ smallest order statistics from the samples $R^\theta_1,\cdots R^\theta_m$.
Our proof follows the technique from \cite{kim_2010}. Let $R^\theta_{(i)}$ is the $i^{th}$ smallest order statistic from the samples $\{R^\theta_i\}_{i=1}^m$. We rewrite \eqref{eq:G} as given below.
\begin{align}
    \label{eq:G2}
    G^m_{R^{\theta}}(x) =
    \begin{cases}
        0,&\textrm{if } x < R^\theta_{(1)}\\
        \frac{i}{m},&\textrm{if } R^\theta_{(i)} \leq x < R^\theta_{(i+1)},\\& i\in\{1,\!\cdots\!,m-1\}\\
        1,&\textrm{if } x \geq R^\theta_{(m)}.\\
    \end{cases}
\end{align}
Let $\nabla l^{\theta}_{(i)}= \sum_{t=0}^{T^{(i)}-1}\nabla\!\log \pi_{\theta}(A_t^{(i)} | S_t^{(i)})$, where $T^{(i)}$ is the length, and  $S_t^{(i)}$ and $A_t^{(i)}$ are the state and action at time $t$ of the episode corresponding to  $R^\theta_{(i)}$. We rewrite \eqref{eq:nabla_G} as given below.
     \begin{align}
     \label{eq:nabla_G2}
     \widehat{\nabla} G^m_{R^{\theta}}(x) =
     \begin{cases}
         0,&\textrm{if } x < R^\theta_{(1)}\\
         \frac{1}{m}\sum_{j=1}^i \nabla l^{\theta}_{(j)},&\textrm{if } R^\theta_{(i)} \leq x < R^\theta_{(i+1)},\\
         & i\in\{1,\!\cdots\!,m-1\}\\
         \frac{1}{m}\sum_{j=1}^m\nabla l^{\theta}_{(j)},&\textrm{if } x \geq R^\theta_{(m)}.
     \end{cases}
 \end{align}
Now,
\begin{align*}
&\widehat{\nabla}^G\!\rho_g(\theta)=-\int_{-M_r}^{M_r} g'(1-G^m_{R^{\theta}}(x)) \widehat{\nabla}G^m_{R^{\theta}}(x)dx\\
&=-\int_{-M_r}^{R^\theta_{(1)}} g'(1-G^m_{R^{\theta}}(x))\widehat{\nabla}G^m_{R^{\theta}}(x)dx\\
&\quad- \sum_{i=1}^{m-1} \int_{R^\theta_{(i)}}^{R^\theta_{(i+1)}} g'(1-G^m_{R^{\theta}}(x))\widehat{\nabla}G^m_{R^{\theta}}(x) dx\\
&\quad- \int_{R^\theta_{(m)}}^{M_r}g'(1-G^m_{R^{\theta}}(x))\widehat{\nabla}G^m_{R^{\theta}}(x)dx\\
&= -\frac{1}{m}\sum\nolimits_{i=1}^{m-1} \int_{R^\theta_{(i)}}^{R^\theta_{(i+1)}} g'\left(1-\frac{i}{m}\right) \sum\nolimits_{j=1}^i \nabla l^{\theta}_{(j)}dx\\
&\quad-\frac{1}{m}\int_{R^\theta_{(m)}}^{M_r}g'_{+}(0)\sum\nolimits_{j=1}^m \nabla l^{\theta}_{(j)} dx\\
&= \frac{1}{m}\sum\nolimits_{i=1}^{m-1} \left( R^\theta_{(i)}-R^\theta_{(i+1)}\right) g'\left(1-\frac{i}{m}\right) \sum\nolimits_{j=1}^i \nabla l^{\theta}_{(j)}\\
&\quad+ \frac{1}{m}\left(R^\theta_{(m)} -M_r\right) g'_{+}(0)\sum\nolimits_{j=1}^m\nabla l^{\theta}_{(j)},
%&= \frac{1}{m}\sum_{i=1}^{m-2} \left( R^\theta_{(i)}-R^\theta_{(i+1)}\right) g'\left(1-\frac{i}{m}\right) \sum_{j=1}^i \nabla l^{\theta}_{(j)}
%-  \frac{1}{m} M_r g'_{+}(0)\sum_{j=1}^m \nabla l^{\theta}_{(j)}.
\end{align*}
where $g'_{+}(0)$ is the right derivative of the distortion function $g$ at $0$.%\hfill{\qed}
\end{proof}
%%%%%%%%%%%%%%%%%%%%%%%%%%%%%%%%%%%%%%%%%%%%%%%%%%%%%%%%%%%%%%%%%%%%%
%\begin{lemma}
%    \label{lm:hat_nabla_rho_H}
%    \begin{align*}
%    &\widehat{\nabla}^H\!\rho_g(\theta)
%     =\frac{1}{m}\sum\nolimits_{i=1}^{m-1} \left( R^b_{(i)}-R^b_{(i+1)}\right) \\
%     &\times g'\left(1\!-\! min\left\{1,\frac{1}{m}\sum\nolimits_{j=1}^{i}\psi^\theta_{(j)}\right\}\right) \sum\nolimits_{j=1}^i \nabla l^{\theta}_{(j)}\psi^\theta_{(j)}\\
%    &+  \frac{1}{m}\left(R^b_{(m)} -M_r\right) g'_{+}(0)\sum\nolimits_{j=1}^m\nabla l^{\theta}_{(j)}\psi^\theta_{(j)}.
%    \end{align*}
%\end{lemma}
%\subsection{Proof of Lemma \ref{lm:hat_nabla_rho_H}}
%\label{subsec:pf_hat_nabla_rho_H}
\begin{proof}(\textbf{Lemma \ref{lm:hat_nabla_rho_H}})
 %Following the proof of Lemma \ref{lm:hat_rho_H}, we rewrite $H^m_{R^{\theta}}(x)$ from \eqref{eq:H} as \eqref{eq:H2}.
%    \begin{align}
%        \label{eq:H2}
%        H^m_{R^{\theta}}(x) =
%        \begin{cases}
%            0,&\textrm{if } x < R^b_{(1)}\\
%            min\{1,\frac{1}{m}\sum\limits_{j=1}^{i}\psi^\theta_{(j)}\},&\textrm{if } R^b_{(i)} \leq x < R^\theta_{(i+1)},\\&i\in\{1,\!\cdots\!,m-1\}\\
%            1,&\textrm{if } x \geq R^b_{(m)},
%        \end{cases}
%    \end{align}
%    where $R^b_{(i)}$ is the $i^{th}$ smallest order statistics from the samples $R^b_1,\cdots R^b_m$, and $\psi^\theta_{(i)}$ is the importance sampling ratio corresponding to $R^b_{(i)}$.
Let $R^b_{(i)}$ be the $i^{th}$ smallest order statistic from the samples $\{R^b_i\}_{i=1}^m$, and $\psi^\theta_{(i)}$ is the importance sampling ratio of $R^b_{(i)}$. Then we rewrite \eqref{eq:H} as given below.
\begin{align}
    \label{eq:H2}
    H^m_{R^{\theta}}(x) =
    \begin{cases}
        0,&\textrm{if } x < R^b_{(1)}\\
        min\{1,\frac{1}{m}\sum\limits_{j=1}^{i}\psi^\theta_{(j)}\},&\textrm{if } R^b_{(i)} \leq x < R^\theta_{(i+1)},\\& i\in\{1,\!\cdots\!,m-1\}\\
        1,&\textrm{if } x \geq R^b_{(m)},
    \end{cases}
\end{align}
Let $\nabla l^{\theta}_{(i)}\!=\! \sum_{t=0}^{T^{(i)}-1}\nabla\!\log \pi_{\theta}(A_t^{(i)} | S_t^{(i)})$, where $T^{(i)}$ is the length, and  $S_t^{(i)}$ and $A_t^{(i)}$ are the state and action at time $t$ of the episode corresponding to  $R^b_{(i)}$. We rewrite \eqref{eq:nabla_H} as given below.
  \begin{align}
    \label{eq:nabla_H2}
      \widehat{\nabla} H^m_{R^{\theta}}(x) =
      \begin{cases}
          0,&\textrm{if } x < R^b_{(1)}\\
          \frac{1}{m}\sum_{j=1}^i \nabla l^{\theta}_{(j)}\psi^\theta_{(j)},&\textrm{if } R^b_{(i)} \leq x < R^b_{(i+1)},\\ &i\in\{1,\!\cdots\!,m-1\}\\
          \frac{1}{m}\sum_{j=1}^m\nabla l^{\theta}_{(j)}\psi^\theta_{(j)},&\textrm{if } x \geq R^b_{(m)}.
      \end{cases}
  \end{align}
Now,
\begin{align*}
    &\widehat{\nabla}^H\!\rho_g(\theta)=-\int_{-M_r}^{M_r} g'(1-H^m_{R^{\theta}}(x)) \widehat{\nabla}H^m_{R^{\theta}}(x)dx\\
    &=-\int_{-M_r}^{R^b_{(1)}} g'(1\!-\!H^m_{R^{\theta}}(x))\widehat{\nabla}H^m_{R^{\theta}}(x)dx\\
   & \quad- \sum\nolimits_{i=1}^{m-1} \int_{R^b_{(i)}}^{R^b_{(i+1)}}g'(1\!-\!H^m_{R^{\theta}}(x))\widehat{\nabla}H^m_{R^{\theta}}(x) dx\\
    &\quad- \int_{R^b_{(m)}}^{M_r}g'(1\!-\!H^m_{R^{\theta}}(x))\widehat{\nabla}H^m_{R^{\theta}}(x)dx\\
    &= -\frac{1}{m}\sum\nolimits_{i=1}^{m-1} \int_{R^b_{(i)}}^{R^b_{(i+1)}}g'\left(1\!-\! min\left\{1,\frac{1}{m}\sum\limits_{j=1}^{i}\psi^\theta_{(j)}\right\}\right) \\
    &\times\sum_{j=1}^i \nabla l^{\theta}_{(j)}\psi^\theta_{(j)}dx
    -  \frac{1}{m}\int_{R^b_{(m)}}^{M_r}g'_{+}(0)\sum\nolimits_{j=1}^m \nabla l^{\theta}_{(j)} \psi^\theta_{(j)}dx\\
    &= \frac{1}{m}\sum_{i=1}^{m-1}\!\left( R^b_{(i)}\!-\!R^b_{(i+1)}\right) g'\!\left(1\!-\! min\left\{1,\frac{1}{m}\sum\limits_{j=1}^{i}\psi^\theta_{(j)}\right\}\right)\\
    &\times\sum_{j=1}^i \nabla l^{\theta}_{(j)}\psi^\theta_{(j)}
    +  \frac{1}{m}\left(R^b_{(m)} -M_r\right) g'_{+}(0)\sum\limits_{j=1}^m\nabla l^{\theta}_{(j)}\psi^\theta_{(j)}.
   % &= \frac{1}{m}\sum_{i=1}^{m-2} \left( R^b_{(i)}-R^b_{(i+1)}\right) g'\left(1\!-\! min\left\{1,\frac{1}{m}\sum_{j=1}^{i}\psi^\theta_{(j)}\right\}\right) \sum_{j=1}^i \nabla l^{\theta}_{(j)}\psi^\theta_{(j)} -  \frac{1}{m} M_r g'(0)\sum_{j=1}^m \nabla l^{\theta}_{(j)}\psi^\theta_{(j)}.
\end{align*}
%\hfill{\qed}
%Here $g'_{+}(0)$ is the right derivative of the distortion function $g$ at $0$.
\end{proof}
%%%%%%%%%%%%%%%%%%%%%%%%%%%%%%%%%%%%%%%%%%%%%%%%%%%%%%%%%%%%%%%%%%%%%%%%%%%%%
\section{Analysis of DRM-offP-LR}
\label{sec:conv_lr_offp}
%\subsection{Analysis of DRM-offP-LR}
\label{subsec:offp_lr}
%\subsubsection{Importance sampling ratio}
%The following result bound the importance sampling ratio of the episodes generated using the behavior policy.
%\begin{lemma}
%    \label{lm:is_ratio}
%    For any episode generated using $b$, the importance sampling ratio $\psi^\theta \leq M_s$ a.s.,$\forall \theta\in\R^d$.
%\end{lemma}
%\begin{proof}
%    From \ref{as:nabla_logpi} and \ref{as:b_pol}, we obtain $\forall \theta \in \R^d,\pi_{\theta}(a|s)>0$ and $b(a|s) >0$, $\forall a \in \mathscr{A}, \textrm{ and } \forall s \in \mathscr{S}$. From \ref{as:b_proper}, we obtain that the episode length is bounded for $b$. So the importance sampling ratio $\psi^\theta$ is bounded for any episode. Hence WLOG, we say $\forall \theta\in\R^d, \psi^\theta \leq M_s$ a.s., for some constant $M_s>0$.
%\end{proof}
%%%%%%%%%%%%%%%%%%%%%%%%%%%%%%%%%%%%%%%%%%%%%%%%%%%%%%%%%%%%%%%%%%%%%%%%%%%%%
%We first prove Lemma \ref{lm:nablaFH}, which expresses the gradient of the CDF $F_{R^{\theta}}(\cdot)$ as an expectation w.r.t the episodes from the behavior policy $b$.
%We replace this expectation by sample average, and forms an estimator for the DRM gradient (see \eqref{eq:nabla_H} and \eqref{eq:hat_nabla_rho_H} in Section \ref{sec:drm}).
%\begin{lemma}
%    \label{lm:nablaFH}
%    $\forall x \in(-M_r,M_r)$,\begin{align*}\nabla F_{R^{\theta}}(x)=\E\left[\1\{R^b\leq x\}\psi^\theta \sum\nolimits_{t=0}^{T-1}\nabla\log \pi_{\theta}(A_t\mid S_t)\right].\end{align*}
%\end{lemma}
%\subsubsection*{Proof of Lemma \ref{lm:nablaFH}}
\begin{proof}(\textbf{Lemma \ref{lm:nablaFH}})
    We use parallel arguments to the proof of Lemma \ref{lm:nablaFG}.
    %    Let $\Omega$ denote the set of all sample episodes. For any episode $\omega\in\Omega$, we denote by $T(\omega)$, its length, and $S_t(\omega)$ and $A_t(\omega)$, the state and action at time $t\in\{0,1,\cdots\}$ respectively.
    %    For an episode $\omega$, let the cumulative discounted reward be
    %    \begin{align*}
        %    R(\omega)=\sum\nolimits_{t=0}^{T(\omega)-1}\gamma^t r(S_t(\omega),A_t(\omega),S_{t+1}(\omega)),\end{align*}
    For any episode $\omega$, let the importance sampling ratio be
    \begin{align}
        \label{eq:psi_omega}
        \psi^\theta(\omega)=\prod\nolimits_{t=0}^{T(\omega)-1}\frac{\pi_{\theta}(A_t(\omega)\mid S_t(\omega))}{b(A_t(\omega)\mid S_t(\omega))}.
    \end{align}
    Also, let
    \begin{align}
        \label{eq:b_omega}
        \p_b(\omega) \!=\!\prod\limits_{t=0}^{T(\omega)-1}b(A_t(\omega)|S_t(\omega))p(S_{t+1}(\omega),S_t(\omega),A_t(\omega)).
    \end{align}
    From \eqref{eq:pi_omega},\eqref{eq:is_grad_onp}, \eqref{eq:psi_omega}, and \eqref{eq:b_omega}, we obtain
    \begin{align}
        \label{eq:is_grad}
        \frac{\nabla \p_\theta(\omega)}{\p_b(\omega)} \!=\!\psi^\theta(\omega)\!\sum\nolimits_{t=0}^{T(\omega)-1}\!\nabla\log\pi_\theta(A_t(\omega) | S_t(\omega)).
    \end{align}
    From \eqref{eq:B_FG}, we obtain
    %\\Let $R(\omega)=\sum_{t=0}^{T(\omega)-1}\gamma^t r(S_t(\omega),A_t(\omega),S_{t+1}(\omega))$ be the cumulative discounted reward, and  $\psi^\theta(\omega)=\prod_{t=0}^{T(\omega)-1}\frac{\pi_{\theta}(A_t(\omega)\mid S_t(\omega))}{b(A_t(\omega)\mid S_t(\omega))}$ be the importance sampling ratio of the episode $\omega$. Let\\
    %$\p_b(\omega) \!=\!\prod_{t=0}^{T(\omega)-1}\!b(A_t(\omega)|S_t(\omega))p(S_{t+1}(\omega),S_t(\omega),A_t(\omega))$. From $\frac{\nabla \p_\theta(\omega)}{\p_b(\omega)} \!=\!\psi^\theta(\omega)\!\sum_{t=0}^{T(\omega)-1}\!\nabla\log\pi_\theta(A_t(\omega) | S_t(\omega))$, and \eqref{eq:B_FG}, we obtain
    \begin{align*}
        &\nabla F_{R^{\theta}}(x)=\sum_{\omega\in\Omega} \1\{R(\omega)\leq x\}\nabla \p_\theta(\omega)\\
        &=\sum_{\omega\in\Omega} \1\{R(\omega)\leq x\} \frac{\nabla \p_\theta(\omega)}{\p_b(\omega)} \p_b(\omega)\\
        &\stackrel{(a)}{=}\!\sum_{\omega\in\Omega}\!\!\1\{R(\omega)\!\leq\! x\} \psi^\theta(\omega)\!\!\!\! \sum_{t=0}^{T(\omega)-1}\!\!\!\!\nabla\log\pi_\theta(A_t(\omega)|S_t(\omega))\p_b(\omega)\\
        &=\E\left[\1\{R^b\leq x\}\psi^\theta\sum\nolimits_{t=0}^{T-1}\nabla\log\pi_\theta(A_t|S_t)\right],
    \end{align*}
    where \((a)\) follows from \eqref{eq:is_grad}.%\hfill\qed
\end{proof}
%%%%%%%%%%%%%%%%%%%%%%%%%%%%%%%%%%%%%%%%%%%%%%%%%%%%%%%%%%%%%%%%%%%%%%%%%%%%%
%\subsubsection{Bias and variance of the DRM gradient estimate}

In the Lemma below, we establish an upper bound on the variance of the gradient estimate $\widehat{\nabla}^H \rho_g(\theta)$ as defined in \eqref{eq:hat_nabla_rho_H}. Subsequently, we use this result to prove Lemma \ref{lm:biasH} and Theorem \ref{tm:drm_offP_lr}.
\begin{lemma}
    \label{lm:varH}
    $\E\left\lVert \widehat{\nabla}^H\! \rho_g(\theta) \right\rVert^2 \leq 4 M_r^2 M_{g'}^2 M_s^2M_e^2M_d^2.$
\end{lemma}
\begin{proof}%(\textbf{Lemma \ref{lm:varH}})
    %    Similar to on-policy case, from \ref{as:b_proper}, we infer that $\E[T]<\infty$. This fact in conjunction with $T\geq0$ implies the following bound:
    %    \begin{align}
        %        \label{eq:M_b}
        %        \exists M_e >0 \textrm{ s.t. } T \leq M_e, \textrm{ a.s}.
        %    \end{align}
    From \eqref{eq:nabla_H}, \ref{as:nabla_logpi}, \eqref{eq:M_b}, and \eqref{eq:is_ratio}, we obtain
    \begin{align}
        \label{eq:nabla_Hm_bound}
        \left\lVert\widehat{\nabla}H^m_{R^{\theta}}(x)\right\rVert^2 \leq M_s^2M_e^2M_d^2 \textrm{ a.s.}, \forall x \in(-M_r,M_r).
    \end{align}
    The result follows by using similar arguments as in Lemma \ref{lm:varG} along with \eqref{eq:nabla_Hm_bound}.
\end{proof}
%%%%%%%%%%%%%%%%%%%%%%%%%%%%%%%%%%%%%%%%%%%%%%%%%%%%%%%%%%%%%%%%%%%%
%Now we prove Lemma \ref{lm:biasH}, which bound the mean squared error of the gradient estimate $\widehat{\nabla}^H \rho_g(\theta)$.
%\begin{lemma}
%    \label{lm:biasH}
%    $\E\left[\left\lVert\widehat{\nabla}^H\! \rho_g(\theta) -\nabla\rho_g(\theta)\right\rVert^2\right]
%    \leq \frac{32 M_r^2M_s^2M_e^2M_d^2(e^2M_{g'}^2 + M_{g''}^2M_s^2)}{m}$.
%\end{lemma}
%\subsubsection*{Proof of Lemma \ref{lm:biasH}}
\begin{proof}(\textbf{Lemma \ref{lm:biasH}})
    We use parallel arguments to the proof of Lemma \ref{lm:biasG}.
    From \eqref{eq:is_ratio}, we obtain $\forall x \in(-M_r,M_r),\left\lvert\1\{R^{\theta}\leq x\}\psi^\theta\right\rvert  \leq  M_s$ a.s., and % and we observe that $\{m\left(\hat{H}^m_{R^{\theta}}(x) - F_{R^{\theta}}(x)\right),m>0\}$ is a set of partial sums of bounded mean zero r.v.s, and hence they are martingales.
    we observe that $\left\{m'\left(\hat{H}^{m'}_{R^{\theta}}(x) - F_{R^{\theta}}(x)\right)\right\}_{m'=1}^{m}$ is a set of partial sums of bounded mean zero r.v.s, and hence they are martingales.
    Using Azuma-Hoeffding's inequality, we obtain $\forall x\in(-M_r,M_r)$,
    \begin{align}
        \label{eq:hatH_prob}
        \p\left(\left\lvert \hat{H}^m_{R^{\theta}}(x) - F_{R^{\theta}}(x) \right\rvert > \epsilon\right) \leq 2\exp\left(-\frac{m\epsilon^2}{2M_s^2}\right).
    \end{align}
    From \eqref{eq:H} and \eqref{eq:hatH}, we observe that\\$\p\left(\left\lvert H^m_{R^{\theta}}(x) - F_{R^{\theta}}(x) \right\rvert \!>\! \epsilon\right) \!\leq\! \p\left(\left\lvert \hat{H}^m_{R^{\theta}}(x) - F_{R^{\theta}}(x) \right\rvert \!>\! \epsilon\right)$. Hence, we obtain $\forall x \in(-M_r,M_r)$,
    \begin{align}
        \label{eq:H_prob}
        \p\left(\left\lvert H^m_{R^{\theta}}(x) - F_{R^{\theta}}(x) \right\rvert > \epsilon\right) \leq 2\exp\left(-\frac{m\epsilon^2}{2M_s^2}\right).
    \end{align}
    Using similar arguments as in \eqref{eq:G_err} along with \eqref{eq:H_prob}, we obtain $\forall x \in(-M_r,M_r)$,
    \begin{align}
        \label{eq:H_err}
        \E\left\lvert H^m_{R^{\theta}}(x) - F_{R^{\theta}}(x)\right\rvert^2 \leq\frac{4 M_s^2}{m}.
    \end{align}
    From \ref{as:nabla_logpi}, \eqref{eq:M_b}, and \eqref{eq:is_ratio}, we obtain $\forall x \in(-M_r,M_r)$,
    \begin{align}
        \label{eq:nabla_H_bound}
        \left\lVert\1\{R^b\leq x\}\psi^\theta \sum_{t=0}^{T-1}\nabla\log \pi_{\theta}(A_t| S_t)\right\rVert \leq  M_s M_e M_d \textrm{ a.s}.
    \end{align}
    We observe that $\left\{m'\left(\widehat{\nabla}H^{m'}_{R^{\theta}}(x) - \nabla F_{R^{\theta}}(x) \right)\right\}_{m'=1}^m$ is a set of partial sums of bounded mean zero r.v.s  from \eqref{eq:nabla_H_bound}, and hence they are martingales. Using vector version of Azuma-Hoeffding inequality from Theorem~1.8-1.9 in \cite{hayes2005large}, we obtain $\forall x \in(-M_r,M_r)$,
    \begin{align}
        \label{eq:nabla_H_prob}
        \p\left(\left\lVert\widehat{\nabla}H^m_{R^{\theta}}(x) \!-\!  \nabla F_{R^{\theta}}(x) \right\rVert \!>\! \epsilon\right) \!\leq\! 2e^2\exp\left(\frac{-m\epsilon^2}{2M_s^2M_e^2M_d^2}\right).
    \end{align}
    Using similar arguments as in \eqref{eq:nabla_G_err} along with \eqref{eq:nabla_H_prob}, we obtain $\forall x \in(-M_r,M_r)$,
    \begin{align}
        \label{eq:nabla_H_err}
        \E\left\lVert \nabla F_{R^{\theta}}(x) - \widehat{\nabla}H^m_{R^{\theta}}(x) \right\rVert^2 = \frac{4e^2M_s^2M_e^2M_d^2 }{m}.
    \end{align}
    Using similar arguments as in  \eqref{eq:G_A}, along with \eqref{eq:H_err}, and \eqref{eq:nabla_Hm_bound}, we obtain $\forall x \in(-M_r,M_r)$,
    \begin{align}
        &\E\left\lVert\left(g'(1-F_{R^{\theta}}(x))-  g'(1-H^m_{R^{\theta}}(x))\right) \widehat{\nabla}H^m_{R^{\theta}}(x) \right\rVert^2\nonumber\\
        &\leq \frac{4M_{g''}^2M_s^4M_e^2M_d^2}{m}.\label{eq:H_A}
    \end{align}
    Using similar arguments as in \eqref{eq:nablaG_A} along with \eqref{eq:nabla_H_err}, we obtain $\forall x \in(-M_r,M_r)$,
    \begin{align}
        \label{eq:nablaH_A}
        &\E\left\lVert g'(1-F_{R^{\theta}}(x)) \left(\nabla F_{R^{\theta}}(x)-\widehat{\nabla}H^m_{R^{\theta}}(x)\right)\right\rVert^2\nonumber\\
        &\leq\frac{4e^2M_{g'}^2M_s^2M_e^2M_d^2}{m}.
    \end{align}
    %Using similar arguments as in \eqref{eq:biasG} along with \eqref{eq:H_A} and \eqref{eq:nablaH_A}, we obtain
    %    \begin{align*}
        %        &\E\left[\left\lVert\widehat{\nabla}^H\! \rho_g(\theta) -\nabla\rho_g(\theta)\right\rVert^2\right]\\
        %        &\leq \frac{32 M_r^2M_s^2M_e^2M_d^2(e^2M_{g'}^2 + M_{g''}^2M_s^2)}{m}.
        %    \end{align*}
    The result follows by using similar arguments as in \eqref{eq:biasG} along with \eqref{eq:H_A} and \eqref{eq:nablaH_A}.%\hfill{\qed}
\end{proof}
%%%%%%%%%%%%%%%%%%%%%%%%%%%%%%%%%%%%%%%%%%%%%%%%%%%%%%%%%%%%%%%%%%%%%%%%%%%%%%%%%%%%%%%%%%%%%%%%%%%%%%%%%%%%%%%
%\subsubsection{Non-asymptotic bound for DRM-OffP-LR}
%Now, we prove Theorem \ref{tm:drm_offP_lr} that presents a non-asymptotic bound for DRM-OffP-LR.
%\subsubsection*{Proof of Theorem \ref{tm:drm_offP_lr}}
\begin{proof}(\textbf{Theorem \ref{tm:drm_offP_lr}})
    By using a completely parallel argument to the initial passage in the proof of Theorem \ref{tm:drm_onP_lr} leading up to \eqref{eq:2_lr_onp}, we obtain
    %Taking expectations on both sides of \eqref{eq:1}, we obtain
    \begin{align*}
        &\alpha\,\E\left \lVert \nabla \rho_g(\theta_k)\right \rVert^2
        \leq  2\,\E\left[\rho_g(\theta_{k+1}) - \rho_g(\theta_{k})\right] \\
        &+ L_{\rho'}\alpha^2\ \E\left\lVert \widehat{\nabla}^H\rho_g(\theta_k) \right\rVert^2
        +  \alpha\E\left \lVert \nabla \rho_g(\theta_k) -\widehat{\nabla}^H\rho_g(\theta_k) \right \rVert^2 \\
        &\leq  2\E\left[\rho_g(\theta_{k+1}) - \rho_g(\theta_{k})\right]\\
        &+ \alpha 4 M_r^2M_s^2M_e^2M_d^2\left(\alpha M_{g'}^2 L_{\rho'}+ \frac{8(e^2M_{g'}^2 + M_{g''}^2M_s^2)}{m}\!\right),
    \end{align*}
    where the last inequality follows from Lemmas \ref{lm:biasH} and \ref{lm:varH}.
    %Summing up \eqref{eq:2H} from $k=0,\cdots,N-1$, we obtain
    Summing the above result from $k\!=\!0,\!\cdots\!,N\!-\!1$, we obtain
    \begin{align*}
        &\alpha\sum\nolimits_{k=0}^{N-1}\E\left \lVert \nabla \rho_g(\theta_k)\right \rVert^2\leq  2 \E\left[\rho_g(\theta_{N}) - \rho_g(\theta_{0})\right]\\
        &+ N \alpha 4 M_r^2M_s^2M_e^2M_d^2 \left(\alpha M_{g'}^2L_{\rho'} \!+\! \frac{8(e^2M_{g'}^2 \!+\! M_{g''}^2M_s^2)}{m}\right).
    \end{align*}
    Since $\theta_R$ is chosen uniformly at random from the policy iterates $\{\theta_0,\cdots,\theta_{N-1}\}$, we obtain
    \begin{align*}
        &\E\left\lVert \nabla \rho_g(\theta_R)\right\rVert^2
        \!=\! \frac{1}{N}\sum\limits_{k=0}^{N-1}\E\left\lVert \nabla \rho_g(\theta_k)\right \rVert^2
        \!\leq\! \frac{2 \left(\rho_g^* \!-\! \rho_g(\theta_{0})\right)}{N \alpha}\\
        &+ 4 M_R^2M_s^2M_e^2M_d^2 \left(\alpha M_{g'}^2L_{\rho'} + \frac{8(e^2M_{g'}^2 + M_{g''}^2M_s^2)}{m}\right).
        %        &\stackrel{(a)}{\leq} \frac{2 \left(\rho_g^* - \rho_g(\theta_{0})\right)}{\sqrt{N}}\\
        %        &\quad+ 4 M_r^2M_s^2M_e^2M_d^2 \left(\frac{M_{g'}^2L_{\rho'}}{\sqrt{N}} + \frac{8(e^2M_{g'}^2 + M_{g''}^2M_s^2)}{m}\right),
        %        &\stackrel{(a)}{\leq} \frac{2 \left(\rho_g^* - \rho_g(\theta_{0})\right)}{\sqrt{N}}\\
        %        &\quad+ \frac{4 M_r^2M_s^2M_e^2M_d^2 \left(M_{g'}^2L_{\rho'} + 8(e^2M_{g'}^2 + M_{g''}^2M_s^2)\right)}{\sqrt{N}},
    \end{align*}
    %   \hfill{\qed}
    %where \((a)\) follows from $\alpha=\frac{1}{\sqrt{N}}$.\hfill{\qed}
    %where \((a)\) follows from $\alpha=\frac{1}{\sqrt{N}}$, and $m=\sqrt{N}$.\hfill{\qed}
\end{proof}

\end{document}